%% file: ms.tex
\documentclass{article}

\PassOptionsToPackage{numbers}{natbib}

    \usepackage[final]{neurips_2021}

\usepackage{hyperref}
\usepackage{url}
\usepackage{booktabs}
\usepackage{amsfonts}
\usepackage{nicefrac}
\usepackage{microtype}
\usepackage{xcolor}
\usepackage{amsmath,amssymb,amsthm,boxedminipage,color}
\usepackage{enumerate,paralist}
\usepackage{aliascnt}
\usepackage[]{cleveref}
\usepackage{graphicx}
\usepackage{bbm}
\usepackage{bm}
\usepackage{array}
\usepackage{stmaryrd}
\usepackage{subcaption}
\usepackage{multirow}
\usepackage{times}
\usepackage{graphicx}
\usepackage{sidecap}
\usepackage{wrapfig}

\input{macros}

\onecolumn

\title{A Theoretical Analysis of Fine-tuning with Linear Teachers}

\author{
  Gal Shachaf \\
  Blavatnik School of Computer Science,\\
  Tel Aviv University, Israel
  \And
  Alon Brutzkus \\
  Blavatnik School of Computer Science,\\
  Tel Aviv University, Israel
  \And
  Amir Globerson \\
  Blavatnik School of Computer Science,\\
  Tel Aviv University, Israel\\
  and Google Research
}

\begin{document}

\maketitle

\begin{abstract}
    \input{abs.tex}
\end{abstract}

\input{intro.tex}

\input{related.tex}

\input{pre.tex}

\input{linear}

\input{deep_linear}

\input{shallow_finetune}

\input{discussion}

\begin{ack}
This work has been supported by the Israeli Science Foundation research grant 1186/18 and the Yandex Initiative for Machine Learning. AB is supported by the Google Doctoral Fellowship in Machine Learning.
\end{ack}
\newpage

\bibliographystyle{unsrt}

\input{ms.bbl}
\newpage
\appendix

\paragraph{Code} In the code used for the experiments we used Pytorch \cite{NEURIPS2019_9015}, Numpy \cite{harris2020array}, SciPy \cite{2020SciPy-NMeth}, and Matplotlib \cite{4160265}.
\input{proof_linear}

\newpage
\input{proof_deep_linear}

\newpage
\input{proof_shallow_finetune}

\end{document}

%% file: macros.tex
\newcommand{\remove}[1]{}
\newcommand{\Draft}[1]{\ifnum\draft=1\texttt{ #1} \fi}

\newcommand{\sdotfill}{\textcolor[rgb]{0.8,0.8,0.8}{\dotfill}} 

\newcommand{\mat}[1]{\mathbf{#1}}
\newcommand{\vect}[1]{\boldsymbol{#1}}
\newcommand{\norm}[1]{\left\|#1\right\|}
\newcommand{\abs}[1]{\left|#1\right|}
\newcommand{\expect}{\mathbb{E}}

\newcommand{\indict}{\mathbb{I}}

\newcommand{\vectorize}[1]{\text{vec}\left(#1\right)}

\newcommand{\comment}[1]{}

\newcommand{\relu}[1]{\sigma\left(#1\right)}

\newcommand{\R}{{\mathbb R}}

\newcommand{\eps}{\varepsilon}

\newcommand{\E}{\mathbb{E}}

\newcommand{\poly}{\operatorname{poly}}

\renewcommand{\cref}{\Cref}

\newtheorem{theorem}{Theorem}[section]
\newtheorem*{theorem*}{Theorem}

\newtheorem{innercustomthm}{Theorem}
\newenvironment{customthm}[1]
  {\renewcommand\theinnercustomthm{#1}\innercustomthm}
  {\endinnercustomthm}
  
\newaliascnt{lemma}{theorem}
\newtheorem{lemma}[lemma]{Lemma}
\aliascntresetthe{lemma}
\crefname{lemma}{Lemma}{Lemmas}

\newaliascnt{assumption}{theorem}
\newtheorem{assumption}[assumption]{Assumption}
\aliascntresetthe{assumption}
\crefname{assumption}{Assumption}{Assumptions}

\newtheorem{innercustomasm}{Assumption}
\newenvironment{customasm}[1]
  {\renewcommand\theinnercustomasm{#1}\innercustomasm}
  {\endinnercustomasm}

\newaliascnt{problem}{theorem}

\aliascntresetthe{problem}
\crefname{problem}{Open Problem}{Open Problems}

\newaliascnt{claim}{theorem}

\aliascntresetthe{claim}
\crefname{claim}{Claim}{Claims}

\newaliascnt{question}{theorem}

\aliascntresetthe{question}
\crefname{question}{Question}{Questions}

\newaliascnt{corollary}{theorem}
\newtheorem{corollary}[corollary]{Corollary}
\aliascntresetthe{corollary}
\crefname{corollary}{Corollary}{Corollaries}
\newtheorem{innercustomcor}{Corollary}
\newenvironment{customcor}[1]
  {\renewcommand\theinnercustomcor{#1}\innercustomcor}
  {\endinnercustomcor}

\newaliascnt{construction}{theorem}

\aliascntresetthe{construction}
\crefname{construction}{Construction}{Constructions}

\newaliascnt{fact}{theorem}

\aliascntresetthe{fact}
\crefname{fact}{Fact}{Facts}

\newaliascnt{proposition}{theorem}

\aliascntresetthe{proposition}
\crefname{proposition}{Proposition}{Propositions}

\newaliascnt{conjecture}{theorem}

\aliascntresetthe{conjecture}
\crefname{conjecture}{Conjecture}{Conjectures}

\newaliascnt{definition}{theorem}
\newtheorem{definition}[definition]{Definition}
\aliascntresetthe{definition}
\crefname{definition}{Definition}{Definitions}

\newaliascnt{remark}{theorem}

\aliascntresetthe{remark}
\crefname{remark}{Remark}{Remarks}

\newaliascnt{proto}{theorem}

\newtheorem{proto}[proto]{Protocol}

\aliascntresetthe{proto}
\crefname{proto}{protocol}{protocols}

\newaliascnt{algo}{theorem}

\aliascntresetthe{algo}
\crefname{algorithm}{algorithm}{algorithms}

\newaliascnt{expr}{theorem}
\newtheorem{expr}[expr]{Experiment}
\aliascntresetthe{expr}
\crefname{experiment}{experiment}{experiments}

\crefname{descriptiona}{description}{descriptions}

\def\FullBox{$\Box$}
\def\qed{\ifmmode\qquad\FullBox\else{\unskip\nobreak\hfil
\penalty50\hskip1em\null\nobreak\hfil\FullBox
\parfillskip=0pt\finalhyphendemerits=0\endgraf}\fi}

\def\qedsketch{\ifmmode\Box\else{\unskip\nobreak\hfil
\penalty50\hskip1em\null\nobreak\hfil$\Box$
\parfillskip=0pt\finalhyphendemerits=0\endgraf}\fi}

\newcommand{\calD}{\mathcal{D}}
\newcommand{\calR}{\mathcal{R}}
\newcommand{\Erc}[1]{\E_{{\bm\eps} \sim \{\pm 1\}^n }\left[ #1 \right]}

\def\vw{\mathbf{w}}
\def\vW{\mathbf{W}}
\def\va{\mathbf{a}}
\def\vZ{\mathbf{Z}}
\def\vX{\mathbf{X}}
\def\vx{\mathbf{x}}
\def\vy{\mathbf{y}}
\def\vbeta{\bm{\beta}}
\def\va{\vect{a}}
\def\vV{\mat{V}}
\def\vU{\mat{U}}
\def\vv{\vect{v}}
\def\vu{\vect{u}}
\def\vLambda{\mat{\Lambda}}
\def\vH{\mat{H}}
\newcommand{\calN}{\mathcal{N}}

\def\vTheta{\mat{\Theta}}
\def\f{f\left(\vx;\vTheta(t)\right)}
\def\vI{\mat{I}}

\def\R{\mathbb{R}}

\def\cF{\mathcal{F}}

\newcommand{\tr}{\mathrm{tr}}

\newcommand{\unif}{\mathrm{unif}}

\newcommand{\bone}[1]{\indict\left\{#1\right\}}

\def\vtheta{\vect{\theta}}
\def\valpha{\vect{\gamma}}
\def\thetas{\vtheta_S}
\def\thetat{\vtheta_T}
\def\thetash{\hat{\vtheta}_S}
\def\thetath{\hat{\vtheta}_T}

\def\Ppar{\mat{P}_{\parallel}}
\def\Pper{\mat{P}_{\perp}}
\def\vbeta{\vect{\beta}}
\def\init{\text{init}}
\def\vu{\vect{u}}
\def\vtildey{\tilde{\vy}}
\def\vSigma{\mat{\Sigma}}
\def\dist{\calD}
\def\lambdas{\lambda_{0_S}}



%% file: abs.tex
Fine-tuning is a common practice in deep learning, achieving excellent generalization results on downstream tasks using relatively little training data. Although widely used in practice, it is lacking strong theoretical understanding. Here we analyze the sample complexity of this scheme for regression with linear teachers in several architectures. Intuitively, the success of fine-tuning depends on the similarity between the source tasks and the target task, however measuring this similarity is non trivial. We show that generalization is related to a measure that considers the relation between the source task, target task and covariance structure of the target data. In the setting of linear regression, we show that under realistic settings a substantial sample complexity reduction is plausible when the above measure is low. For deep linear regression, we present a novel result regarding the inductive bias of gradient-based training when the network is initialized with pretrained weights. Using this result we show that the similarity measure for this setting is also affected by the depth of the network. We further present results on shallow ReLU models, and analyze the dependence of sample complexity on source and target tasks in this setting. 

%% file: intro.tex
\section{Introduction}

In recent years fine-tuning has emerged as an effective 
approach to learning tasks with relatively little labeled data. In this setting, a model is first trained on a source task where much data is available (e.g., masked language modeling for BERT), and then it is further tuned using gradient descent methods on labeled data of a target task \cite{devlin2019bert, lee2019latent, geva2020injecting, gururangan2020don}. Furthermore, it has been observed that fine-tuning can outperform the strategy of fixing the representation learned on the source task, mainly in natural language processing \cite{devlin2019bert, peters2019tune}. 
Despite its empirical success, fine-tuning is poorly understood from a theoretical perspective. One apparent conundrum is that fine-tuned models can be much larger than the number of target training points, resulting in a heavily overparameterized model that is prone to overfitting and poor generalization. Thus, the answer must lie in the fact that fine-tuning is performed with gradient descent and not an arbitrary algorithm that could potentially ``ignore'' the source task \cite{ZhangBHRV17}. Here we set out to formalize this problem and understand the factors that determine whether fine-tuning will succeed. We note that this question can be viewed as part of the general quest to understand the implicit bias of gradient based methods \cite{ZhangBHRV17,lyu2019gradient, woodworth2020kernel, chizat2020implicit, wu2021direction, moroshko2020implicit, arora2019implicit, sarussi2021towards}, but in the particular context of fine-tuning.

We begin by highlighting the obvious link between fine-tuning and initialization. Namely, the only difference between ``standard'' training of a target task and fine-tuning on it, is the initial value of the model weights before beginning the gradient updates. Our goal is to understand the interplay between the model parameters at initialization (namely the source task), the target distribution, and the accuracy of the fine-tuned model. A natural hypothesis is that the distance between the pretrained and fine-tuned model weights is what governs the success of fine-tuning. Indeed, some argue that this is both the key to bound the generalization error of a model and the implicit regularization of gradient-based methods \cite{nagarajan2019generalization, li2020gradient, neyshabur2019towards, dziugaite2017computing}. However, this approach has been discouraged both by empirical testing of the generalization bounds inspired by it \cite{jiang2019fantastic} and by theoretical works showing this cannot be the inductive bias in deep neural networks \cite{razin2020implicit}. Our results further establish the hypothesis that the success of fine-tuning is affected by other factors.

In this paper we focus on the case in which both source and target regression tasks are linear functions of the input. We start by considering one layer linear networks, and derive novel sample complexity results for fine-tuning. We then proceed to the more complex case of deep linear networks, and prove a novel result characterizing the fine-tuned model as a function of both the weights after pretraining and the depth of the network, and use it to derive corresponding generalization results. 

Our results provide several surprising insights. First, we show that the covariance structure of the target data has a significant effect on the success of fine-tuning. In particular, sample complexity is affected by the degree of alignment between the source-target weight difference and the eigenvectors of the target covariance. Second, we find a strong connection between the depth of the network and the results of the fine-tuning process, since deeper networks will serve to cancel the effect of scale differences between source and target tasks. Our results are corroborated by empirical evaluations.

We conclude with results on ReLU networks, providing the first sample complexity result for fine-tuning. For the case of linear teachers, this asserts a simple connection between the source and target models and the test error of fine-tuning.

Taken together, our results demonstrate that fine-tuning is affected not only by some notion of distance between the source and target tasks, but also by the target covariance and the architecture of the model. These results can potentially lead to improved accuracy in this setting via appropriate design of the tasks used for pretraining and the choice of the model architecture.

%% file: related.tex
\section{Related work}

Empirical work \cite{neyshabur2020being} has shown that two instances of models initialized from pre-trained weights are more similar in features space than those initialized randomly.
Other works \cite{chua2021fine, du2021fewshot, mcnamara2017risk} have shown that fine-tuned models generalize well when the representation used by the target task is  similar to the one used by the source tasks.

In linear regression, \cite{gunasekar2018characterizing} showed that gradient descent finds the solution with minimal distance to the initial weights. More recently, attention has turned towards the phenomenon of ``benign overfitting'' \cite{bartlett2020benign,hastie2019surprises} in high dimensional linear regression, where despite fitting noise in training data, population risk may be low. Theoretical analysis of this setting \cite{bartlett2020benign} studied how it is affected by the data covariance structure. Benign overfitting was also recently analyzed in the context of ridge-regression \cite{tsigler2020benign} and online stochastic gradient descent \cite{zou2021benign}. Our work continues this line of work on high dimensional regression, but differs from the above papers as we start from a source task, then train on a fixed training set from a target task and consider the global optimum of the this training loss (unlike online SGD). Furthermore, we go beyond the linear regression framework, and obtain surprising characteristics of fine-tuning in deep linear networks.

For linear regression with deep linear models, \cite{azulay2021implicit} have recently shown an implicit bias for a two-layer network with deterministic initialization, and \cite{yun2020unifying} have shown an implicit bias for a network with arbitrary depth and near-zero random initialization. Our work generalizes the inductive bias found by \cite{azulay2021implicit} to a network of arbitrary depth, and analyses the generalization error of such networks for infinite depth. For linear regression with shallow linear networks \cite{arora2019fine} have shown a generalization bound that depends only on the norm of the target task, which we use in Section \ref{sec:shallow_finetune}.

%% file: pre.tex
\section{Preliminaries and settings}\label{sec:pre}

\paragraph{Notations} Let $\|\cdot\|$ be the $L^2$ norm for vectors and the spectral norm for matrices. For a vector $\vv$ we denote $\hat{\vv} \triangleq \frac{\vv}{\norm{\vv}}$. For a matrix $\mat{M} \in \R^{d \times d}$ and some $0 \le m \le d$, we define $\mat{M}_{\le m} \in \R^{d \times m}$ to be the matrix containing the first $m$ columns of $\mat{M}$. Similarly, we let $\mat{M}_{>m}$ denote the matrix containing the columns from $m+1$ to $d$ in $\mat{M}$.

Let $\dist$ be a distribution over $\R^d$. Let $\mat{\Sigma}$ be the covariance matrix of $\dist$ and let $\mat{V\Lambda V}^\top$ be its eigenvalue decomposition such that $\lambda_1 \ge \ldots \ge \lambda_d$. We define the projection matrices:
\begin{align*}
    \mat{P}_{\le k} \triangleq \mat{V}_{\le k}\mat{V}_{\le k}^\top; \quad \mat{P}_{> k} \triangleq \mat{V}_{> k}\mat{V}_{> k}^\top,
\end{align*}
projecting onto the span of the top $k$ eigenvectors of $\vSigma$, onto the span of the $d-k$ bottom eigenvectors of $\vSigma$, respectively. We will refer to the former as the ``top-$k$ span'' of $\vSigma$, and to the latter as the ``bottom-$k$ span'' of $\vSigma$.

Let $\mat{X} \in \R^{n \times d}$ be the row matrix of $n<d$ samples drawn from $\dist$, and denote the empirical covariance matrix $\tfrac{1}{n}\vX^T \vX$ by $\tilde{\vSigma}$. Define $\mat{P}_{\parallel}$ to be the projection matrix into the row space of $\mat{X}$, and $\mat{P}_{\perp}$ to be the projection matrix into its orthogonal complement, i.e.:
\begin{align*}
    \Ppar \triangleq \mat{X}^\top(\mat{XX}^\top)^{-1}\mat{X}, \quad
    \Pper \triangleq \mat{I} - \Ppar.
\end{align*}

Consider a set of parameters $\vTheta$, and let $\vTheta(t)$ denote the set of parameters at time $t$. We denote the output of a model whose weights are $\vTheta(t)$ on a vector $\vx$ by $\f \in \R$. In the different sections of this work we will overload $f$ with different architectures.

We consider the problem of fine-tuning based transfer learning in regression tasks with linear teachers.
Let $\thetat \in \R^{d}$ be the ground-truth parameters of the target task, i.e. the linear teacher which we wish to learn, and $\vy \in \R^{n}$ be the target labels of $\mat{X}$, s.t. $\vy = \vX \thetat$.

We define $L(\vTheta)$ to be the empirical MSE loss on $\vX, \vy$ and define $R(\vTheta)$ as the $\dist$ population loss:
\begin{align*}
    L\left(\vTheta\right) \triangleq \frac{1}{n}\norm{f\left(\vX, \vTheta\right) - \vy}_2^2, \quad R(\vTheta) \triangleq \E_{\vx \sim \dist}\left[\left({\vx^\top\vect{\theta}_T - f\left(\vx, \vTheta\right)}\right)^2\right].
\end{align*}

We separate the training procedure into two parts. In the first ``pretraining'' part, we train a model on $n_S$ pretraining samples $\vX_S \in \R^{n_S \times d}$ labeled by a linear teacher $\thetas$ (i.e., $\vy_S = \vX_S \thetas \in \R^{n_S}$), resulting in the set of model weights $\vTheta_S$. In the second part, which we call fine-tuning, we initialize a model with the pretrained weights $\vTheta(0) = \vTheta_S$ and learn the target task by optimizing $L(\vTheta(t))$. 

Optimization is done by either gradient descent (GD) or gradient flow (GF). Let $\vtheta(t)$ be some weight vector or weight matrix in $\vTheta(t)$. The dynamics for gradient descent optimization with some learning rate $\eta >0$ are $
\vtheta(t+1) = \vtheta(t) -\eta \frac{\partial L\left(\vTheta(t)\right)}{\partial \vtheta(t)}$, and the dynamics for gradient flow are ${\dot\vtheta(t)} = - \frac{\partial L\left(\vTheta(t)\right)}{\partial \vtheta(t)}$. Next we state several assumptions about our setup.

\begin{assumption} \label{asm:non-singular}
    $\mat{X X^T}$ is non-singular. i.e. the rows of $\vX$ are linearly-independent.
\end{assumption}
This assumption holds with high probability for, e.g., a continuous distribution with support over a non-zero measure set. This assumption is only used for simplicity, as the high probability can be incorporated into the analysis.  

\begin{assumption}[Perfect pretraining] \label{asm:pretrain}
    The pretraining optimization process learns the linear teacher perfectly, e.g. for linear regression we assume that $
        f\left(\vx, \vTheta_S\right) = \vx^\top \thetas$, for $\vx \sim \dist$.
\end{assumption}
Notice that for linear and deep linear models, perfect pretraining can be achieved when $n_S \ge d$. Our results can be easily extended to the case where the equality $f\left(\vx, \vTheta_S\right) = \vx^\top \thetas$ holds approximately and with high probability, but for simplicity we assume equality.

\begin{assumption}[Zero train loss] \label{asm:train_loss}
    The fine-tuning converges, i.e. $\lim_{t \to \infty} L\left(\vTheta(t)\right) = 0.$
\end{assumption}
We note that when $f$ is standard linear regression, arbitrarily small train loss can be obtained via gradient descent. For deep linear networks, it can be shown \cite{arora2018convergence} that under suitable initialization a global optimum can be reached, and thus \cref{asm:train_loss} holds for this framework as well.

%% file: linear.tex
\section{Analyzing fine-tuning in linear regression} \label{sec:linear}
In this section we analyze fine-tuning for the case of linear teachers for linear regression when using gradient descent for optimization. We define $\vTheta(t) = \vw(t) \in \R^d$ and overload $f(\vx, \vTheta(t)) \triangleq \vx^\top \vw(t)$. In what follows we denote
the parameter learned in the fine-tuning process by $\valpha \triangleq \lim_{t \to \infty} \vw(t)$.

\subsection{Results}

The following known results (e.g.,  \cite{gunasekar2018characterizing, bartlett2020benign, wu2021direction}) show the inductive bias of gradient descent with non-zero initialization in under-determined linear regression and the corresponding population loss. 
\begin{theorem}\cite{gunasekar2018characterizing, bartlett2020benign, wu2021direction}\label{thm:linear_inductive_pop_risk}
    When $f(\vx, \vTheta)$ is a linear function, fine-tuning with GD under \cref{asm:non-singular}, \cref{asm:pretrain} and \cref{asm:train_loss} results in the following model:
    \begin{align} \label{eq:linear_sol}
        \valpha = \mat{P}_{\perp}\thetas + \mat{P}_{\parallel}\thetat,
    \end{align}
    and
    \begin{align} \label{eq:linear_risk}
        R(\valpha) = \norm{\vSigma^{\nicefrac{1}{2}}\mat{P}_{\perp}\left(\thetat - \thetas\right)}^2.
    \end{align}
\end{theorem}
\cref{thm:linear_inductive_pop_risk} provides two interesting observations: the first is that $\valpha$ consists of two parts, one which is the projection of the initial weights $\thetas$ into the null space of $\vX$, and the other which is the projection of $\thetat$ into the span of $\vX$. The second observation is that the population risk depends solely on the difference $\thetat - \thetas$ that is projected to the null space of the data. For completeness, the proof of \cref{thm:linear_inductive_pop_risk} is given in the supplementary.

\cref{thm:linear_inductive_pop_risk} depends on the data matrix $\vX$ (via $\mat{P}_{\parallel},\mat{P}_{\perp}$). However, to better understand the properties of fine-tuning, a high probability bound on $R$ that does not depend on $\vX$ is desirable. We provide such a bound, highlighting the dependence of the population risk on the source and target tasks, and the target covariance $\Sigma$.
\begin{theorem} \label{thm:linear_upper}
    Assume the conditions of \cref{thm:linear_inductive_pop_risk} hold, and
    assume that the rows of $\vX$ are i.i.d. subgaussian centered random vectors. Then, there exists a constant $c>0$, such that, for all $\delta \ge 1$, and for all $1 \le m \le d$ such that $\lambda_m > 0$, with probability at least $1-e^{-\delta}$ over $\vX$, the population risk $R(\valpha)$ is bounded by:
    \begin{align} \label{eq:thm_upper_cauchy}
        2g(\vect{\lambda}, \delta, n)^{3}\frac{\|\mat{P}_{\le m}(\thetat - \thetas)\|^2}{\lambda_{m}^2} + 2g(\vect{\lambda}, \delta, n)\|\mat{P}_{> m}(\thetat - \thetas)\|^2,
    \end{align}
    where $
         g(\vect{\lambda}, \delta, n) = c \lambda_1\max\{\sqrt{\tfrac{\sum_i \lambda_i}{n\lambda_1}},\tfrac{\sum_i \lambda_i}{n\lambda_1}, \sqrt{\tfrac{\delta}{n}}, \tfrac{\delta}{n}\}
    $  and $\norm{\tilde{\vSigma} - \vSigma} \le g(\vect{\lambda}, \delta, n)$.
\end{theorem}

In the proof, we address the randomness of $\Pper(\thetat - \thetas)$ in \eqref{eq:linear_risk}, by decomposing $\thetat - \thetas$ into its top-$k$ span and bottom-$k$ span components, and then applying the Davis-Kahan sin($\Theta$) theorem \cite{davis1970rotation} to bound the norm of the projection of the former to the null space of the data. The full proof is given in the supp.

The bound in \cref{thm:linear_upper} has two key components. The first is the function $g(\vect{\lambda},\delta,n)$ that captures how well the covariance $\vSigma$ is estimated, and shows the dependence of the bound on the number of train samples used (as it depends on $n^{-0.5}$). The second relates to the two matrix norms of $\thetat - \thetas$ with respect to different parts of the covariance $\vSigma$. Notice that the term relating to the \textit{top-k} span decreases like $n^{-1.5}$, while the term relating to \textit{bottom-k} span decreases like $n^{-0.5}$.

This theorem highlights the conditions under which fine-tuning is expected to perform well. For small enough $n$ s.t. $g(\vect{\lambda},\delta,n) > 1$, the bound mainly depends on $\|\mat{P}_{\le m}(\thetat - \thetas)\|$. In this case, the bound will be low if $\thetat$ and $\thetas$ are close in the span of the \textit{top} eigenvectors of the target distribution. On the other hand, for large enough $n$ s.t. $g(\vect{\lambda},\delta,n) < 1$, the bound mainly depends on $\|\mat{P}_{> m}(\thetat - \thetas)\|$. Thus, the bound will be low if $\thetat$ and $\thetas$ are close in the span of the \textit{bottom} eigenvectors of the target distribution.

We conclude with a remark regarding the integer $m$ appearing in the bound, in the case where $g(\vect{\lambda},\delta,n) < 1$. While finding the exact $m$ that minimizes the bound is not straightforward, the trade-off in selecting it suggests taking the largest $m$ which holds $\lambda_{m+1} \approx \lambda_m$. This will ``cover'' more of $\mat{P}_{>m}\left(\thetat - \thetas\right)$ without greatly increasing the left part of \eqref{eq:thm_upper_cauchy}. 

\subsection{Experiments}
In Figure \ref{fig:linear_synthetic} we empirically verify the conclusions from the bound in \eqref{eq:thm_upper_cauchy}. We set $d=1000$ and design the target covariance $\vSigma$ s.t. the first $m=50$ eigenvalues are significantly larger than the rest (1.5 vs. 0.3).
We then consider two settings for $\thetat - \thetas$. In the first, which we call ``Top Eigen Align'', we select $\thetat$ and $\thetas$ such that $\mat{P}_{\le m}(\thetat - \thetas) = 0$. In the second which we call ``Bottom Eigen Align'' we set $\mat{P}_{> m}(\thetat - \thetas) = 0$. In both settings we use the same norm $\|\thetat - \thetas\|_2$, to show that the bound is not affected by this norm. 

As discussed above, our bound suggests better generalization performance of ``Bottom Eigen Align'' for large $n$ and better performance of ``Top Eigen Align'' for small $n$. Indeed, we see that while for very few samples ``Top Eigen Align'' has a lower population loss than ``Bottom Eigen Align'', the population loss of "Bottom Eigen Align" drops significantly as $n$ grows, and drops to zero well before $n=d$. 

\comment{
The conclusion on the preferred alignment of $\thetat$ and $\thetas$ from \cref{thm:linear_upper} is shown in $\cref{fig:linear_synthetic}$. For $d=1000$, we design the target covariance s.t. the first $m=50$ eigenvalues are much larger than the rest (1.5 vs. 0.3). We initialize two predictors: "Top Eigen Align" is initialized with $\mat{P}_{\le m}(\thetat - \thetas) = 0$ and "Bottom Eigen Align" is the linear predictor initialized with $\mat{P}_{> m}(\thetat - \thetas) = 0$. We also initialize the two predictors s.t. the norm of the difference from each one of them to $\thetat$ is approximately the same. As discussed above, our bound suggests better generalization for "Bottom Eigen Align" as $n$ grows, since its convergence depends more on it. Indeed, we see that while for very samples "Top Eigen Align" has a lower population loss than "Bottom Eigen Align", the population loss of "Bottom Eigen Align" drops significantly as $n$ grows, and drops to zero well before $n=d$. The better generalization of "Top Eigen Align" for very few samples is also explained by this dependence on the estimation of $\vSigma$, since for a small $n$ this factor can be much larger than 1, which effects "Bottom Eigen Align" more than "Top Eigen Align".
}

We next evaluate the bound on fine-tuning tasks taken from the MNIST dataset \cite{lecun-mnisthandwrittendigit-2010}, and compare it to alternative bounds. Specifically, since we do not expect bounds to be numerically accurate, we calculate the correlation between the actual risk in the experiment and the risk predicted by the bounds. The task we consider (both source and target) is binary classification, which we model as regression to outputs $\{-1,+1\}$. We generate $K$ source-target  task pairs (e.g., source task is label $2$ vs label $3$ and target tasks is label $5$ vs label $6$). For each such pair we perform source training followed by fine-tuning to target. We then record both the 0-1 error on an independent test set and the value predicted by the bounds. This way we obtain $K$ pairs of points (i.e., actual error vs bound), and calculate the $R^2$ for these pairs, indicating the level to which the bound agrees with the actual error. In addition to our bound in \eqref{eq:thm_upper_cauchy}, we consider the following: the norm of source-target difference $\norm{\thetat - \thetas}^2$ and a bound adapted from \cite{bartlett2020benign} to the case of fine-tuning.\footnote{The adaptation is straightforward: since the population loss for non-random initialization depends on $\thetat - \thetas$ instead of $\thetat$, we can replace the ground-truth expression $\vtheta^\star$ in Theorem 4 from \cite{bartlett2020benign} with $\thetat - \thetas$.} The results in Table \ref{tab:MNIST_linear} show that there is a strong correlation between our bound and the actual error, and the correlation is weaker for the other bounds.
\comment{
\cref{tab:MNIST_linear} demonstrates the efficacy of \eqref{eq:thm_upper_cauchy} as a predictor of accuracy in various transfer tasks in MNIST \cite{lecun-mnisthandwrittendigit-2010}. We choose 5 non-overlapping sets of labels and train a 1-accuracy predictor on each set. We then construct 20 transfer experiments, such that at each experiment we pick two sets, initialize the weights with the weights of the predictor of the source task and train it on $k$ train samples from the train set of the target (labeled by the "ground truth" of the target task). We then evaluate the fine-tuned predictor on the test set of the target. We run each experiment 25 times for different sampled train sets, and then calculate the correlation between the mean accuracy of each experiment the bounds given by $\norm{\thetat - \thetas}^2$, a bound adapted to the case of fine-tuning from \cite{bartlett2020benign}, and \cref{thm:linear_upper}. We run this overall experiment 10 times and show the mean and standard deviation of each $R^2$ value. We clearly see that there is a strong correlation between the connection between $\thetat - \thetas$ and $\vSigma$, as shown in our bound, to the generalization of fine-tuning.
}
\begin{table}
  \caption{Correlation coefficient $R^2$ between the accuracy on different transfer tasks in MNIST and various population risk upper bounds. Each value is a mean over 10 calculations of $R^2$ with different initialization, and each $R^2$ is calculated from $20$ points, each one representing a mean accuracy value of 25 random samples.}
  \label{tab:MNIST_linear}
  \centering
  \begin{tabular}{llllll}
    \toprule
    Number of Samples & 10 & 15 & 20 & 25 & 30  \\
    \midrule
    $\norm{\thetat - \thetas}^2$ & 0.69 $\pm$ 0.03 & 0.68 $\pm$ 0.04 & 0.66 $\pm$ 0.04 & 0.64 $\pm$ 0.03 & 0.62 $\pm$ 0.02 \\
    Bound from \cite{bartlett2020benign} & 0.73 $\pm$ 0.03 & 0.75 $\pm$ 0.03 & 0.74 $\pm$ 0.03 & 0.71 $\pm$ 0.02 & 0.67 $\pm$ 0.02\\
    Ours for $m=2$ & \textbf{0.86 $\pm$ 0.02}  & \textbf{0.89 $\pm$ 0.02} & \textbf{0.84 $\pm$ 0.02} & \textbf{0.75 $\pm$ 0.01} & \textbf{0.69 $\pm$ 0.02}\\
    \bottomrule
  \end{tabular}
\end{table}

\begin{figure}[ht]
\centering
\begin{subfigure}{0.55\textwidth}
  \centering
  \includegraphics[width=1\linewidth]{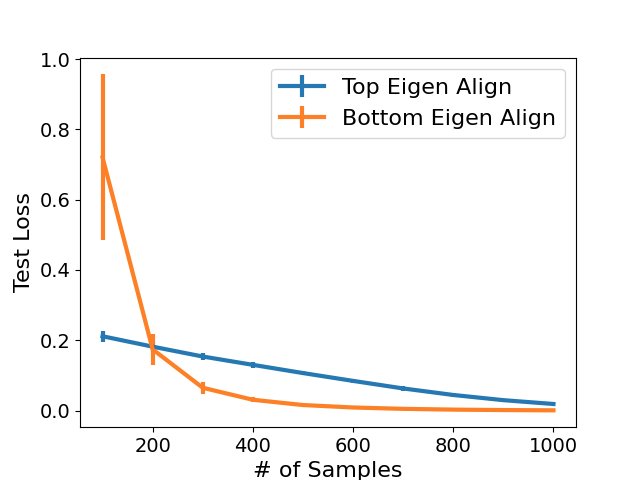}
\end{subfigure}
\caption{Comparison between different $\thetat - \thetas$. "Top Eigen Align" is the linear predictor initialized with $\mat{P}_{\le m}(\thetat - \thetas) = 0$ and "Bottom Eigen Align" is the linear predictor initialized with $\mat{P}_{> m}(\thetat - \thetas) = 0$, for $m$=50. The top $m$ eigenvalues have the value 1.5, compared to the rest which have the value 0.3.}
\label{fig:linear_synthetic}
\end{figure}

%% file: deep_linear.tex
\section{Analyzing fine-tuning in deep linear networks} \label{sec:deep_linear}
In this section we focus on the setting of overparameterized deep linear networks. Although the resulting function is linear in its inputs, like in the previous section, we shall see that the effect of fine-tuning is markedly different. Previous works (e.g. \cite{arora2018optimization,ji2019gradient}) have shown that linear networks exhibit many interesting properties which make them a good study case towards more complex non-linear networks.

We consider networks with $L$ layers, given by the following matrices: $\vTheta(t) = \{\vW_1(t), \cdots, \vW_L(t)\}$ s.t. $\vW_j(t) \in \R^{d_{j-1} \times d_{j}}$, $d_0 = d$, $d_L = 1$ and for $1 \le j \le L-1: d_j \ge d$. We also define:
\begin{align*}
    &\vbeta(t) = \vW_1(t) \cdot \vW_2(t) \cdots \vW_{L}(t),
\end{align*}
such that $\f(t) = \vx^\top \vbeta(t)$. From \cref{asm:pretrain}, we have that $\vbeta(0) = \thetas$.

We recall the condition of perfect balancedness (or 0-balancedness) \cite{arora2018convergence}:
\begin{definition}
    The weights of a depth $L$ deep linear network at time $t$ are called 0-balanced if:
    \begin{align} \label{eq:balancedness}
        \vW_j(t)^\top \vW_j(t) = \vW_{j+1}(t) \vW_{j+1}(t)^\top \quad \text{for} \quad j \in [L-1].
    \end{align}
\end{definition}

Our analysis requires the initial random initialization (prior to pretraining) to be 0-balanced, which can be achieved with a near zero random initialization, as discussed in \cite{arora2018convergence}.
We provide three results on the effect of fine-tuning in this setting. The first result shows the inductive bias of fine-tuning a depth $L$ deep linear network (\cref{thm:deep_linear_solution}), which holds for arbitrary $L$ and generalizes known results for $L=1$ (\cref{thm:linear_inductive_pop_risk}) and $L=2$ \cite{azulay2021implicit}. The second result analyzes the population risk of such a predictor when $L \rightarrow \infty$ for certain settings (\cref{thm:deep_rescaled} and \cref{thm:deep_infty_rand_proj_generalization}). The third result shows why fixing the first layer (or any set of layers containing the first layer) after pretraining can harm fine-tuning (\cref{thm:deep_freeze}).

The next theorem characterizes the model learned by fine-tuning in the above setting (it can thus be viewed as the deep-linear version of the $L=1$ result in \cref{thm:linear_inductive_pop_risk}):
\begin{theorem} \label{thm:deep_linear_solution}
    Assume that before pretraining, the weights of the model were 0-balanced and that \cref{asm:non-singular}, \cref{asm:pretrain} and \cref{asm:train_loss} hold. Then:
    \begin{align} \label{eq:deep_L_sol}
        \lim_{t \to \infty} \vbeta(t) = \left(\frac{\|\lim_{t \to \infty}\vbeta(t)\|}{\norm{\thetas}}\right)^{\frac{L-1}{L}}\Pper\thetas + \Ppar\thetat
    \end{align}
    and:
    \begin{align} \label{eq:deep_infty_sol}
        \lim_{L \to \infty}\lim_{t \to \infty} \vbeta(t) = \frac{\|\Ppar\thetat\|}{\|\Ppar\thetas\|}\Pper\thetas + \Ppar\thetat.
    \end{align}
\end{theorem}
To prove this, we focus on $\vW_1$, and notice that the gradients $\dot{\vW_1}(t)$ are in the span of $\vX$, and hence $\Pper\vW_1(0)$ and its norm remain static during the GF optimization (\cite{yun2020unifying}). We then analyze the norm of the fine-tuned model by using the 0-balancedness property of the weights and the min-norm solution to the equivalent linear regression problem, and achieve \eqref{eq:deep_L_sol}. \eqref{eq:deep_infty_sol} is achieved by calculating the limit w.r.t. $L$. The proof of \cref{thm:deep_linear_solution} is given in the supplementary.

Although the expression in \eqref{eq:deep_L_sol} is not a closed form expression for $\lim_{t \to \infty} \vbeta(t)$ (because $\norm{\lim_{t \to \infty} \vbeta(t)}$ appears on the RHS), taking $L$ to infinity \eqref{eq:deep_infty_sol} does result in a closed form expression and demonstrates the effect of increasing model depth. As in \eqref{eq:linear_sol}, we see that the end-to-end equivalent has two components: one which is parallel to the data and one which is orthogonal to it. However, while in \eqref{eq:linear_sol} the orthogonal component has the original norm of the orthogonal projection of $\thetas$, the expression in \eqref{eq:deep_infty_sol} offers a re-scaling of the norm of this component by some ratio that also depends on $\thetat$. Presenting this phenomenon for the infinity depth limit might look impractical, but the empirical results given in this section show that the effect of depth is apparent even for models of relatively small depth.

\subsection{When Does Depth Help Fine-Tuning?}
In this subsection we wish to understand the effect of depth on the population risk of the fine-tuned model.
For simplicity we focus on the limit in \eqref{eq:deep_infty_sol}, and denote $\vbeta = \lim_{L \to \infty}\lim_{t \to \infty} \vbeta(t)$.

Since the linear network is a linear function of $\vect{x}$, we can derive an expression for the population risk of the network, similar to \eqref{eq:linear_risk}:
\begin{align} \label{eq:deep_population}
    R(\vbeta) = \left\| \mat{\Sigma}^{\frac{1}{2}}\Pper\left(\thetat - \frac{\|\Ppar\thetat\|}{\|\Ppar\thetas\|}\thetas\right) \right\|^2.
\end{align}
However, since $\Ppar$ depends on the random matrix $\vX$, without further assumptions this expression by itself is not enough to understand the behaviour of $R(\vbeta)$. \cref{thm:deep_rescaled} and \cref{thm:deep_infty_rand_proj_generalization} analyze cases for which a bound on \eqref{eq:deep_population} can be achieved, showing that it depends on $\norm{\thetat}(\thetath - \thetash)$, i.e. the product of the norm of $\thetat$ and the difference of the {\em normalized} $\thetat$ and $\thetas$, compared to \eqref{eq:linear_risk} which depends on the difference between the un-normalized vectors. This observation further highlights the fact that the distance between source and target vectors is not a good predictor of fine-tuning accuracy for some architectures, as fine-tuning can still succeed even if the source and target are very far as long as they are aligned.

We formalize this in the following result, where $\thetat$ is identical to $\thetas$ in direction, but not in norm.
\begin{theorem}\label{thm:deep_rescaled}
Assume that the conditions of \cref{thm:deep_linear_solution} hold, and that $\thetath = \thetash$. Namely:
\begin{align*}
    \thetat = \alpha \thetas, \quad \text{for} \quad \alpha > 0,
\end{align*}
then for $L \rightarrow \infty$ the risk of the end-to-end solution $\vbeta$ is
\begin{align*}
    R(\vbeta) = 0,
\end{align*}
while for the $L=1$ solution $\valpha$, the risk  is:
\begin{align} \label{eq:rescaled_linear}
    R(\valpha) = \left(\frac{\alpha-1}{\alpha}\right)^2\|\vSigma^{\nicefrac{1}{2}}\mat{P}_{\perp}\thetat\|^2 \neq 0 \quad \text{for} \quad \alpha \neq 1, \alpha > 0.
\end{align}
\end{theorem}
This setting highlights our conclusion on the role of alignment in deep linear models: if the tasks are aligned, the deep linear predictor achieves zero generalization even with a single sample, while the population risk of the $L=1$ predictor still depends on $n$.

Another example for this behaviour can be seen when $\vX$ is i.i.d Gaussian (i.e., $\dist = \calN(0, 1)^d$).
\begin{theorem}\label{thm:deep_infty_rand_proj_generalization}
    Assume that the conditions of \cref{thm:deep_linear_solution} hold, and let $\vX \sim \calN(0,1)^{d}$. Suppose $n\le d$, then there exists a constant $c>0$ such that for any $\epsilon > 0$ with probability at least $1 - 4\exp(-c\epsilon^2 n) -4\exp\left(-c\epsilon^2 (d-n)\right)$ the population risk for the $L \to \infty$ end-to-end predictor $\vbeta$ is bounded as follows:
    \begin{align}\label{eq:rand_proj_deep}
    R(\vbeta) 
    &\le \frac{d-n}{d}(1+\epsilon)^2\norm{\thetat}^2\norm{\thetath - \thetash}^2 + \frac{d-n}{d}\zeta(\norm{\thetat})^2,
\end{align}
    for $\zeta(\norm{\thetat}) \approx \epsilon\norm{\thetat}$.
    For the $L=1$ linear regression solution $\valpha$ this risk is bounded by
    \begin{align}\label{eq:rand_proj_lin}
        R(\valpha) \le \frac{d-n}{d}(1+\epsilon)^2\norm{\thetat - \thetas}^2.
    \end{align}
\end{theorem}
The above result is a direct analysis of  \eqref{eq:deep_population}  when $\vSigma = \mat{I}$ by using Lemma 5.3.2 from \cite{vershynin2018high} to analyze the effects of $\Ppar, \Pper$. Comparing \eqref{eq:rand_proj_deep} and \eqref{eq:rand_proj_lin}, we see that while \eqref{eq:rand_proj_lin} depends on the distance between the two un-normalized tasks, \eqref{eq:rand_proj_deep} depends on the norm of the target task and the alignment of the tasks, but not at all on the norm of the source task. The proofs of \cref{thm:deep_rescaled} and \cref{thm:deep_infty_rand_proj_generalization} are given in the supp.

\begin{figure*}[ht]
\centering
\begin{subfigure}{.4\textwidth}
  \includegraphics[width=1\linewidth]{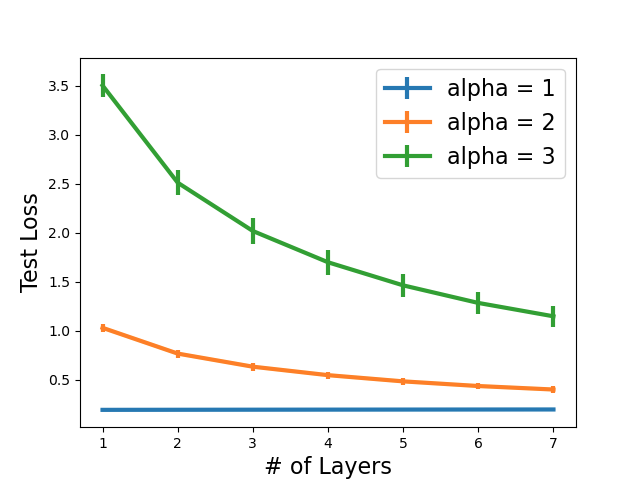}
  \caption{}
\label{fig:deep_linear_scale}
\end{subfigure}
\begin{subfigure}{.4\textwidth}
    \includegraphics[width=1\linewidth]{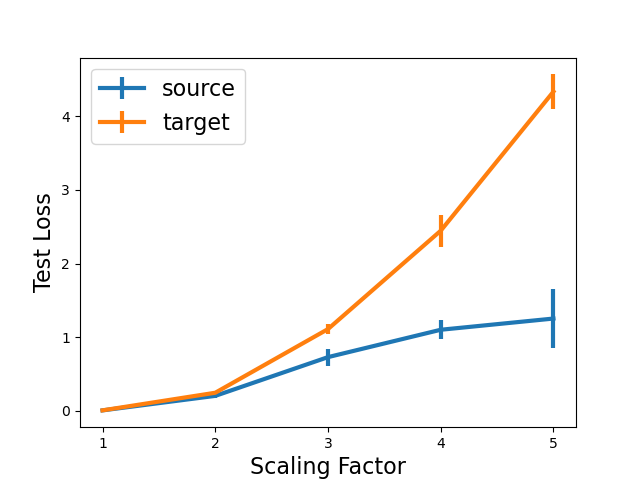}
    \caption{}
    \label{fig:deep_source_vs_target}
\end{subfigure}
\caption{(a) The effect of depth on fine-tuning when $\thetat$ is a $\alpha$ scaled, $\epsilon$ noised version of $\thetas$ with $\nicefrac{d}{10}$ samples. (b) The effect of changing the scale of either source weights or target weights in a 7-layers model.}
\end{figure*}

\subsection{Deep linear fine-tuning with fixing the first layer(s)}
A common trick when performing fine-tuning is to fix, or ``freeze'' (i.e. not train), the first $k$ layers of a model during the optimization on the target task. This method reduces the risk of over-fitting these layers to the small training set.\footnote{This over-fitting is sometimes referred to as ``catastrophic forgetting'' of the source task.} The next theorem shows that for deep linear networks this method degenerates the training process.
\begin{theorem} \label{thm:deep_freeze}
    Assume the setting of \cref{thm:deep_linear_solution}. Then, if we freeze the first layer (or any number $k$ of first layers) during fine-tuning, the fine-tuned model will be given by
        $\langle\vbeta(t), \vect{x}\rangle = c \langle \vx, \thetas \rangle$, for some constant $c$.
\end{theorem}

The key idea in the proof is to show that the product of the $k$ first layers is equal to $\thetas$ up to a scaling factor, which is a result of \cite{yun2020unifying}.
The result implies that after fine-tuning the model is still equal to the source task, independently of the target task. Thus, fine-tuning essentially fails completely, and its error cannot be reduced with additional target data.

\begin{figure*}[ht]
\centering
\includegraphics[width=0.55\linewidth]{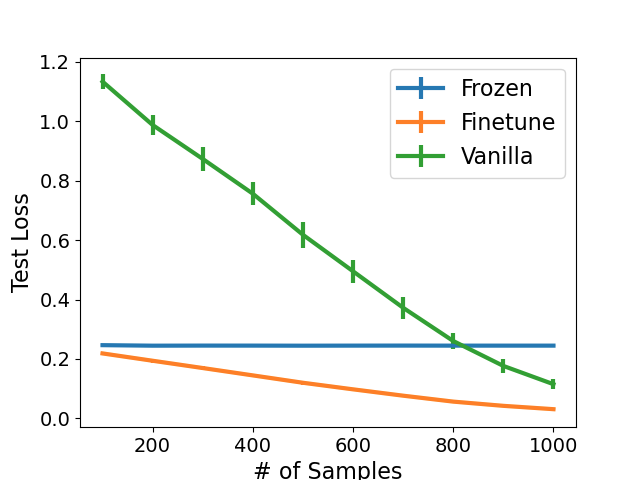}
\caption{A network whose first layer is fixed has a constant generalization loss due to degeneration effect in \cref{thm:deep_freeze}.}
\label{fig:deep_linear_frozen}
\end{figure*}

This result is achieved under the assumption of $0$-balancedness prior to pretraining, which happens e.g. when initializing the weights with an infinitesimally small variance, as this property leads to the degeneracy of the output of the frozen k-layers. Though the proof of \cref{thm:deep_freeze} depends on this $0$-balancedness property of the network, the experiments shown in \cref{fig:deep_linear_frozen} were conducted with a small initialization scale, that is not guaranteed to result in $0$-balancedness, but rather in $\delta$-approximate balancedness \cite{arora2018convergence} when $\delta$ is small. These experiments show empirically that the phenomenon of learning failure is observed even when $\delta > 0$. Intuitively, this is because the effective rank of the weight matrices is close to one, and thus learning the second layer is an ill-conditioned problem, which leads to slower convergence and can prevent the model from fine-tuning on the target data with a constant gradient step.

A possible workaround to this failure of learning would be to initialize the weights prior to pretraining with a larger scale of initialization (e.g. with Xavier \cite{pmlr-v9-glorot10a}), thus increasing the rank of each layer and preventing degeneracy. Pre-training with multiple source tasks (as suggested in e.g. \cite{du2021fewshot}) may also help the fine-tuning optimization.

\subsection{Experiments}

We next describe experiments that support the results in this section. \cref{thm:deep_rescaled} 
predicts that deeper nets will successfully learn a case where source and target vectors are aligned, but with different norms. This is demonstrated in \cref{fig:deep_linear_scale} where source and target tasks are related via $\thetat = \alpha \thetas + \vect{\epsilon}$, where $\vect{\epsilon}$ is a standard Gaussian vector whose norm is approximately $0.5\norm{\thetas}$.
It can be seen that when  $\alpha \approx 1$, there is no difference between models of different depth. However, as $\alpha$ increases, adding depth has a positive effect on fine-tuning accuracy. \cref{thm:deep_infty_rand_proj_generalization} predicts that the test loss for a deep linear model would depend only on the alignment of $\thetas$ and $\thetat$ (i.e. $\norm{\thetath - \thetash})$ and on the  $\norm{\thetat}$, but not on $\norm{\thetas}$. This is demonstrated in \cref{fig:deep_source_vs_target} where source and target task are initialized s.t. $\norm{\thetath - \thetash} \approx 0.1$. In each experiment, either $\thetat = \alpha \thetath$ or $\thetas = \alpha \thetash$, where $\alpha$ is the ``Scaling Factor'', and the other has norm of 1. It can be seen that increasing the norm of the target vector harms generalization much more than increasing the norm of the source vector, as the theorem predicts, even for a relatively shallow model.

\cref{thm:deep_freeze} states that fixing the first layer in deep linear nets can result in failure to fine-tune. We illustrate this empirically in \cref{fig:deep_linear_frozen}, where we compare three two-layer linear models on the same target task: 1) A ``Frozen'' model that fixes the first layer after pretraining. 2) A ``Vanilla'' model that trains the network from scratch on the target, ignoring the source pre-training. 3) A ``Finetune'' model that first trains on source and fine-tunes to target. As predicted by theory, the "frozen" model's performance is poor, and fine-tuning has better sample complexity. 

%% file: shallow_finetune.tex
\section{Analyzing fine-tuning in shallow ReLU networks} \label{sec:shallow_finetune}

Analyzing optimization and generalization in non-linear networks is challenging. However, analysis in the Neural Tangent Kernel (NTK) regime is sometimes simpler \cite{du2018gradient, arora2019fine}. Thus, here we take a first step towards understanding fine-tuning in non-linear networks by analyzing this problem in the NTK regime. Specifically, we consider the setting of a two-layer ReLU  network with $m$ neurons in the hidden layer. Hence, we consider $\vTheta(t) = \{\vW(t), \va\}$ and $\f = \frac{1}{\sqrt{m}}\sum_{r=1}^{m}{a_r \sigma(\vx^\top\vw_r(t)})$
where $\sigma$ is the ReLU function, $\vw_1(t), \ldots, \vw_m(t) \in \R^d$, the rows of $\vW(t)$, are vectors in the first layer, and $\va \in \{-1, 1\}^m$ is the vector of weights in the second layer. We initialize  $\va$ uniformly and fix it during optimization as in \cite{du2018gradient}. Before pretraining, the first layer parameters are initialized from a standard Gaussian with variance $\kappa^2$. We also assume that $\norm{\vx}=1$ for all $\vx$ samples from $\dist$. We let $f\left(\vX, \vTheta\right) \in \mathbb{R}^n$ be the vector of predictions of $f$ on the data $\vX$.

For the next theorem we do not assume linear teachers, and instead assume an arbitrary labeling function $g_S$ such that $\vy_S = g_S(\vX_S)$, for $\vX_S \in \R^{n_S \times d}, \vy_S \in \R^{n_S}$ the pretraining data and labels, respectively. We also assume that $\vy = g_T(\vX)$ for some arbitrary function $g_T$. For simplicity, we assume $|y|_i \le 1$ for  $i \in [n]$.
We consider a setting where the pretraining phase is done using a two-layer network in the NTK regime, under the assumptions of Theorem 4.1 from \cite{arora2019fine} with respect to the variables $m$, $\kappa$, $\eta$ and sufficiently many iterations.\footnote{See the supp for a bound on the number of iterations.} Next, in the fine-tuning phase, we train a network initialized with the weights given by the pretraining phase. We use the same value of $m$ for the fine-tuning phase. We rely on the analysis given in \cite{du2018gradient, arora2019fine} and achieve an upper bound on the population risk of the fine-tuned model:

\begin{theorem}\label{thm:shallow_relu_generalization}
	Fix a failure probability $\delta \in (0, 1)$. We assume that \cref{asm:non-singular} holds. Suppose $\kappa = O\left( \frac{ \lambda_0 \delta}{ n} \right)$, $m\ge \kappa^{-2} \poly\left(n, n_S, \lambda_0^{-1}, \delta^{-1} \right)$.
	Consider any loss function $\ell: \R\times\R \to [0, 1]$ that is $1$-Lipschitz in the first argument such that $\ell(y, y)=0$.  
	Then with probability at least $1-\delta$,\footnote{Over the random initialization of the pretraining network.}  the two-layer neural network $f(\cdot, \vTheta(t))$ fine-tuned by GD for $t\ge \Omega\left( \frac{1}{\eta\lambda_0} \log{\norm{\vtildey}_2^{-1}} \right)$ iterations has population loss:
	\begin{equation} \label{eqn:main_generalization}
        R\left(\vTheta(t)\right) \le 2\sqrt{\frac{\vtildey^\top \left(\mat{H}^{\infty}\right)^{-1}\vtildey}{n}}  + O\left( \sqrt{\frac{\log\frac{n}{\lambda_0\delta}}{n}} \right),
	\end{equation}
    for $\vtildey \equiv  \vy - f\left(\vX, \vTheta(0)\right)$.
\end{theorem}
The above result shows that the true risk of the fine-tuned model is related to the distance of learned outputs $\vy$ from the outputs after pretraining $f\left(\vX, \vTheta(0)\right)$. The proof of \cref{thm:shallow_relu_generalization} is given in the supp.

As in previous NTK regime analyses, this result holds when the weights of the fine-tuned model do not ``move'' too far away from the weights at random initialization. Thus, the proof approach is to bound the distance between the Gram matrix $\vH(t)$ and the infinite-width gram matrix $\vH^\infty$ with a decreasing function in $m$. The main challenge is that the weights $\vW(0)$ are not initialized i.i.d as described above. To address this we provide a careful analysis of the dynamics and show that $\vH(t)$ is close to $\vH$ at random initialization, even when considering the pretraining phase, which in turn is close to $\vH^\infty$.

We next apply our results to the case of linear source and target tasks. We thus assume that $g_S, g_T$ are linear functions with parameters $\thetas, \thetat$.  For simplicity of exposition we assume $f\left(\vx, \vTheta(0)\right) = \vx^\top \thetas$ exactly (Assumption \ref{asm:pretrain}). Before bounding the risk of fine-tuning we bound the RHS of \eqref{eqn:main_generalization} in the linear case:
\begin{corollary} \label{cor:linear_transfer}
    Suppose that $g_S(\vX) \triangleq \vX^\top \thetas$, $g_T(\vX) \triangleq \vX^\top \thetat$,
and assume Assumption \ref{asm:pretrain} holds. Then, $\sqrt{\tilde{\vect y}^\top (\mat H^\infty)^{-1}\tilde{\vect y }} \le  3 \norm{\thetat - \thetas}_2.$
\end{corollary}
This is a direct corollary of Theorem 6.1 from \cite{arora2019fine} on $\vtildey$ defined above.
\cref{thm:shallow_relu_generalization} and \cref{cor:linear_transfer} result in the a bound on the risk of the fine-tuned model:
\begin{corollary}
    Under the conditions of \cref{thm:shallow_relu_generalization} and \cref{cor:linear_transfer}, it holds that
    \begin{align*}
        	R(\vTheta(t)) \le \frac{6 \norm{\thetat - \thetas}_2}{\sqrt{n}}   + O\left( \sqrt{\frac{\log\frac{n}{\lambda_0\delta}}{n}} \right).
    \end{align*}
\end{corollary}
We note that fine-tuning is improved as the distance between source and target decreases.
In our analysis of linear networks (\cref{thm:linear_upper} and  \cref{thm:deep_infty_rand_proj_generalization}) we obtained a more fine-grained result depending on the covariance structure. We conjecture that the non-linear case will have similar results, which will likely involve the covariance structure in the NTK feature space.

%% file: discussion.tex
\section{Discussion}

This paper gives a fine-grained analysis of the process of fine-tuning with linear teachers in several different architectures. It offers insights into the inductive bias of gradient-descent and the implied relation between the source task, the target task and the target covariance that is needed for this process to succeed. We believe our conclusions pave a way towards understanding why some pretrained models work better than others and what biases are transferred from those models during fine-tuning.

A limitation of our work is the simplicity of the models analyzed, and it would certainly be interesting to extend these. Our setting deals only with linear teachers, and assumes the label noise to be zero. Furthermore, we only show upper bounds on the population risk, and not matching lower bounds. For deep linear networks we assume a certain initialization which is less standard than normalized initializers such as Xavier. For non-linear models, we analyze the simple model of a shallow ReLU network, and only in the NTK regime.

An interesting direction to explore is formulating a bound similar to \cref{thm:linear_upper} for regression in the RKHS space given by the NTK, where the covariance is now over the RKHS space and thus more challenging to analyze. Another interesting setting is classification with exponential losses. Since the classifier learned by GD in this case has diverging norm, it is not clear how fine-tuning is beneficial, although in practice it often is. We leave these questions for future work.

%% file: proof_linear.tex
\section{Proofs for linear regression}\label{appendix:linear}


This appendix includes proofs for Section 4. It starts by analyzing the solution achieved by applying gradient descent on a linear regression problem with non-zero initialization, and shows its exact population risk. Then, this risk is bounded  from above by using concentration bounds to bound various aspects of the difference between the true target covariance and the estimated target covariance.


Recall the assumptions:
\begin{customasm}{3.1}[Main Text]\label{asm:non-singular_supp}
    $\mat{X X^T}$ is non-singular. i.e. the rows of $\vX$ are linearly-independent.
\end{customasm}

\begin{customasm}{3.2}[Main Text]\label{asm:pretrain_supp}
    The pretraining optimization process learns the linear teacher perfectly, e.g. for linear regression we assume that $
    f\left(\vx, \vTheta_S\right) = \vx^\top \thetas$, for $\vx \sim \dist$.
\end{customasm}

\begin{customasm}{3.3}[Main Text]\label{asm:train_loss_supp}
    The fine-tuning converges, i.e. $\lim_{t \to \infty} L\left(\vTheta(t)\right) = 0.$
\end{customasm}

\subsection{Proof of Theorem 4.1}\label{subsec:proof_41}


As mentioned in the main text, both parts of the theorem have been proven before \cite{gunasekar2018characterizing, bartlett2020benign, wu2021direction}. The proof is provided for completeness, and can be skipped.

\begin{lemma} \label{lm:gd_solution_linear}
    Assume \cref{asm:train_loss_supp}, and that there exists some vector $\vw \in \R^d$ s.t. $\vy = \vX \vw$ (i.e. the data is generated via a linear teacher), then the solution achieved by using GD with initialization $\vtheta_0$ in order to minimize:
    \begin{align} \label{eq:mse}
        \min_{\vtheta \in \R^{d}} \tfrac{1}{2}\|\vX\vtheta -\vy\|_2^2.
    \end{align}
    is
    \begin{align}\label{eq:linear_gd_sol}
        \vtheta^\star = \Pper\vtheta_0 + \Ppar\vw.
    \end{align}
\end{lemma}
\begin{proof}
First, observe that the gradient step for this problem is
\begin{align*}
    \vtheta_{t+1} = \vtheta_t + \eta \vX^T(\vy - \vX \thetat).
\end{align*}
Hence, all of the steps are in the span of $\vX^T$, and GD converges to a solution of the form:
\begin{align*}
    \vtheta^\star = \vtheta_0 + \vX^T \va
\end{align*}
for some $\va \in R^{n}$. The vector $\vtheta^\star$ must also achieve a loss of zero in \cref{eq:mse} (because we know that $\vw$ achieves a loss of zero, and GD minimizes this objective). Therefore:
\begin{align*}
    \vX\vtheta^\star &= \vy \\
    \vX(\vtheta_0 + \vX^T \va) &= \vy\\
    \vX\vX^T \va &= \vy - \vX\vtheta_0\\
    \va &\stackrel{1}{=} (\vX\vX^T)^{-1} (\vy - \vX\vtheta_0) \\
    \Rightarrow \vtheta^\star &= \vtheta_0 + \vX^T(\vX\vX^T)^{-1} (\vy - \vX\vtheta_0),
\end{align*}
with (1) due to \cref{asm:non-singular_supp}.

Replacing $\vy$ with $\vX \vw$, and by using the definitions of $\Ppar$ and $\Pper$ from Section 3, it follows that
\begin{align*}
    \vtheta_0 + \vX^T(\vX\vX^T)^{-1} (\vy - \vX\vtheta_0) &= \vtheta_0 + \vX^T(\vX\vX^T)^{-1} (\vX \vw - \vX\vtheta_0) \\
    &= \left(\vI - \vX^T(\vX\vX^T)^{-1}\vX\right)\vtheta_0 + \vX^T(\vX\vX^T)^{-1}\vX\vw\\
    &= \Pper \vtheta_0 + \Ppar \vw.
\end{align*}
\end{proof}

We can now prove the theorem.
\begin{proof}[Proof of Theorem 4.1 (Main Text)]

The proof for Eq.1 in the main text is straightforward by using \cref{lm:gd_solution_linear} with $\vtheta_0 = \thetas$ and $\vw=\thetat$.

As for Eq.2 in the main text, by \cref{lm:gd_solution_linear} it follows that

\begin{align*}
    \valpha = \Pper \thetas + \Ppar \thetat.
\end{align*}
Since $\Ppar + \Pper = \vI$ it follows that
\begin{align*}
    R(\valpha) &= \E_{\vx \sim \dist}\left[\left({\vx^\top\vect{\theta}_T - \f}\right)^2\right] = \E_{\vx\sim \dist}\left[\left(\vx^\top\left(\thetat - \Pper \thetas - \Ppar \thetat\right)\right)^2\right]\\
    &= \E_{\vx\sim \dist}\left[\left(\vx^\top\Pper\left(\thetat - \thetas \right)\right)^2\right] = \E_{\vx\sim \dist}\left[\left(\thetat - \thetas \right)^T\Pper\vx\vx^\top\Pper\left(\thetat - \thetas \right)\right]\\
    &= \left(\thetat - \thetas \right)^T\Pper\E_{\vx\sim \dist}\left[\vx\vx^\top\right]\Pper\left(\thetat - \thetas \right)
    = \left(\thetat - \thetas \right)^T\Pper^T\vSigma\Pper\left(\thetat - \thetas \right)\\
    &= \norm{\vSigma^{0.5}\Pper\left(\thetat - \thetas \right)}^2.
\end{align*}
thus concluding the proof.
\end{proof}

\subsection{Proof of Theorem 4.2: Upper bound of the population risk for linear regression}


Recall the Davis-Kahan $sin(\Theta)$ theorem:
\begin{theorem}[\cite{davis1970rotation}] \label{thm:davis-kahan}
    Let $A = E_0 A_0 E_0^T +E_1 A_1 E_1^T$ and $A+H = F_0 \Lambda_0 F_0^T + F_1 \Lambda_1 F_1^T$ be symmetric matrices with $[E_0, E_1]$ and $[F_0, F_1]$ orthogonal. If the eigenvalues of $A_0$ are
contained in an interval $(a, b)$, and the eigenvalues of $\Lambda_1$ are excluded from the interval $(a-\delta, b+\delta)$ for some $\delta > 0$, then
\begin{align}
    \|F_1^T E_0\| \le \frac{\|F_1^T H E_0 \|}{\delta}
\end{align}
for any unitarily invariant norm $\|\cdot\|$.
\end{theorem}

The following theorem is a concentration bound on the difference between the true and estimated covariance matrices: $\norm{\vSigma - \tilde{\vSigma}}$:
\begin{theorem}[Theorem 9 from \cite{koltchinskii2017concentration}]\label{thm:concentration}
Let $X,X_1,\ldots,X_n$ be i.i.d. weakly square integrable centered random vectors in $E$ with covariance operator $\vSigma.$ If $X$ is subgaussian and pregaussian, then there exists a constant 
$c>0$ such that, for all $\delta\geq 1,$ with probability at least $1-e^{-\delta},$ 
 \begin{align*}
&
\|\tilde\vSigma - \vSigma\|\leq 
c\|\vSigma\|\max\left\{\sqrt{\frac{r(\vSigma)}{n}},
      \frac{r(\vSigma)}{n}, \sqrt{\frac{\delta}{n}}, \frac{\delta}{n} \right\} \triangleq g(\vect{\lambda},\delta,n),
\end{align*}
where
  \[
    r(\vSigma) := \frac{(\E\|x\|)^2}{\|\vSigma\|}
    \leq \frac{\tr(\vSigma)}{\|\vSigma\|}
    = \frac{\sum_i \lambda_i}{\lambda_1}.
  \]
\end{theorem}

The following lemma uses \cref{thm:davis-kahan} to upper bound the dot product between the $d-n$ bottom eigenvectors of the estimated covariance and the top $k$ eigenvectors of the target covariance:
\begin{lemma} \label{lem:davis-kahan}
    For all  $1 \le k \le d$ such that $\lambda_k > 0$ it holds that:
    \begin{align*}
        \norm{\tilde{\vV}_{>n}^T \vV_{\le k}} \le \tfrac{\norm{\tilde{\vSigma} - \vSigma}}{\lambda_k}
    \end{align*}
\end{lemma}
\begin{proof}

In order to use \cref{thm:davis-kahan} with $\delta = \lambda_k$ to bound $\| \tilde{\vV}_{>n}^T \vV_{\le k} \|$, one must show that the conditions of \cref{thm:davis-kahan} are met. Let $\mat{A} = \vSigma$, $\mat{A} + \mat{H} = \tilde{\vSigma}$, $\mat{E}_0 = \vV_{\le k}$,  $\mat{A}_0 = \vLambda_{\le k}$, $ \mat{F}_1 = \tilde{\vV}_{> n}$, and $\vLambda_1 = \tilde{\vLambda}_{> n}$. Notice that $\vX$ is a rank-$n$ matrix, and so is the estimated covariance $\tilde{\vSigma}$, hence it bottom $d-n$ eigenvalues are zero. Thus, all of the $d-n$ eigenvalues of $\vLambda_1$ equal zero. Also, recall that the eigenvalues of $\vSigma$ are in descending order.
Thus, all of the eigenvalues of $\mat{A}_0$ are in the interval $(\lambda_k, \lambda_1)$ and all of the eigenvalues of $\vLambda_1$ (which equal 0) are excluded from the interval $(0, \lambda_1 + \lambda_k)$. Hence the conditions of \cref{thm:davis-kahan} are met and for $\delta  = \lambda_k$:
\begin{align*}
    \|\tilde{\vV}_{>n} ^T \vV_{\le k}\| &\le \frac{\|\tilde{\vV}_{>n}^T(\tilde{\vSigma} - \vSigma)\vV_{\le k}\|}{\lambda_k}\\
    &\stackrel{(1)}{\le} \frac{\|\tilde{\vV}_{>n}\| \|\tilde{\vSigma} - \vSigma\|  \|\vV_{\le k}\|}{\lambda_k}\\
    &\stackrel{(2)}{=} \frac{\|\tilde{\vSigma} - \vSigma\|}{\lambda_k},
\end{align*}
with (1) due to Cauchy-Schwartz inequality, (2) due to $\tilde{\vV}_{>n}$, $\vV_{\le k}$ being orthonormal matrices, which concludes the proof.
\end{proof}
We can now prove the theorem.
\begin{proof}[Proof of Theorem 4.2 (Main Text)]
Let $\tilde{\vU}\tilde{\mat{\Gamma}}\tilde{\vV}^T$ be the singular value decomposition of $\vX$ such that $\tilde{\vU} \in \R^{n\times n}, \tilde{\vV} \in \R^{d \times d}$ are unitary matrices and let $\tilde{\vv}_i$ be the $i$-th column of $\tilde{\vV}$.

First, notice that $\Ppar = \vX^\top(\vX\vX^\top)^{-1}\vX$ can be also written as $\vI - \tilde{\vV}_{>n}\tilde{\vV}_{>n}^T$:
\begin{align*}
    \vX^\top(\vX\vX^\top)^{-1}\vX &= \tilde{\vV}\tilde{\mat{\Gamma}}^\top \tilde{\vU}^{T}(\tilde{\vU}\tilde{\mat{\Gamma}} \tilde{\vV}^{T}\tilde{\vV}\tilde{\mat{\Gamma}}^\top \tilde{\vU}^{T})^{-1}\tilde{\vU} \tilde{\mat{\Gamma}} \tilde{\vV}^{T}\\
    &\stackrel{(1)}{=} \tilde{\vV}\tilde{\mat{\Gamma}}^\top \tilde{\vU}^{T}(\tilde{\vU}(\tilde{\mat{\Gamma}}\tilde{\mat{\Gamma}}^\top) \tilde{\vU}^{T})^{-1}\tilde{\vU} \tilde{\mat{\Gamma}} \tilde{\vV}^{T}\\
    &= \tilde{\vV}\tilde{\mat{\Gamma}}^\top \tilde{\vU}^{T}(\tilde{\vU}(\tilde{\mat{\Gamma}}\tilde{\mat{\Gamma}}^\top)\tilde{\vU}^{T})^{-1}\tilde{\vU} \tilde{\mat{\Gamma}} \tilde{\vV}^{T}\\
    &\stackrel{(2)}{=} \tilde{\vV}\tilde{\mat{\Gamma}}^\top \tilde{\vU}^{T} \tilde{\vU}(\tilde{\mat{\Gamma}}\tilde{\mat{\Gamma}}^\top)^{-1} \tilde{\vU}^T \tilde{\vU} \tilde{\mat{\Gamma}} \tilde{\vV}^{T}\\
    &\stackrel{(3)}{=} \tilde{\vV}\tilde{\mat{\Gamma}}^\top (\tilde{\mat{\Gamma}}\tilde{\mat{\Gamma}}^\top)^{-1} \tilde{\mat{\Gamma}} \tilde{\vV}^{T}
    = \tilde{\vV}\cdot \textbf{diag}\mathbf{(1_{1:n}, 0_{n+1:d})} \cdot \tilde{\vV}^{T}\\
    &= \sum_{i=1}^{n}{\tilde{\vv}_i\cdot \tilde{\vv}_i^T} = \sum_{i=1}^{d}{\tilde{\vv}_i\cdot \tilde{\vv}_i^T} - \sum_{i=n+1}^{d}{\tilde{\vv}_i\cdot \tilde{\vv}_i^{T}}\\
    &\stackrel{(4)}{=} \vI - \sum_{i=n+1}^{d}{\tilde{\vv}_i\cdot \tilde{\vv}_i^{T}} = \vI - \tilde{\vV}_{>n}\tilde{\vV}_{>n}^T.
\end{align*}
Where (1),(3),(4) are due to $\tilde{\vU}$, $\tilde{\vV}$ being unitary, and (2) is due to $\tilde{\vU}(\tilde{\mat{\Gamma}}\tilde{\mat{\Gamma}}^\top)\tilde{\vU}^{T}(\tilde{\vU}(\tilde{\mat{\Gamma}}\tilde{\mat{\Gamma}}^\top)^{-1}\tilde{\vU}^{T}) = \vI$.

From Eq.2 in the main text it follows that:

\begin{align*}
    R(\valpha) &= \norm{\vSigma^{0.5}\Pper\left(\thetat - \thetas \right)}^2 \\
    &= (\thetat - \thetas)^T \tilde{\vV}_{>n}\tilde{\vV}_{>n}^T \vSigma\tilde{\vV}_{>n}\tilde{\vV}_{>n}^T(\thetat - \thetas)\\
    &= (\thetat - \thetas)^T \tilde{\vV}_{>n}\tilde{\vV}_{>n}^T \vV \vLambda \vV^T \tilde{\vV}_{>n}\tilde{\vV}_{>n}^T(\thetat - \thetas),
\end{align*}

Notice that $\Pper\tilde{\vSigma}\Pper = 0$, as was shown in \cite{bartlett2020benign}:
\begin{align*}
    \Pper\tilde{\vSigma} &= \Pper \tilde{\vV} \tilde\vLambda \tilde{\vV}^T = \Pper \left(\tilde{\vV}_{\le n} \tilde\vLambda_{\le n} \tilde{\vV}_{\le n}^T + \tilde{\vV}_{>n} \tilde\vLambda_{>n} \tilde{\vV}_{>n}^\top \right)\\
    &=  \tilde{\vV}_{>n} \tilde{\vV}_{>n}^T \tilde{\vV}_{\le n} \tilde\vLambda_{\le n} \tilde{\vV}_{\le n}^T + \tilde{\vV}_{>n}\tilde{\vV}_{>n}^T \tilde{\vV}_{>n} \tilde\vLambda_{>n} \tilde{\vV}_{>n}^T \stackrel{(1)}{=} 0
\end{align*}
where (1) is due to $\tilde{\vV}_{>n}$, $\tilde{\vV}_{\le n}$ being orthogonal and $\tilde{\lambda}_j= 0, \forall j > n$. 

Then:
\begin{align}
    R(\valpha) &= 
    \left(\thetas - \thetat\right)^\top\Pper\vSigma\Pper\left(\thetas - \thetat\right)  \nonumber \\
    &= \left(\thetas - \thetat\right)^\top\Pper\left(\vSigma - \tilde{\vSigma}\right)\Pper\left(\thetas - \thetat\right)  \nonumber\\
    &= \norm{\left(\vSigma - \tilde{\vSigma}
    \right)^{0.5}\Pper\left(\thetas - \thetat\right)}^2  \nonumber \\
    &\le \norm{\vSigma - \tilde{\vSigma}}\norm{\Pper\left(\thetas - \thetat\right)}^2 \label{eq:linear_upper_cauchy},
\end{align}
where the last inequality is due to the Cauchy-Schwartz inequality.

The next step in the proof is to bound $ \norm{\Pper\left(\thetas - \thetat\right)}^2$. We start by bounding $\norm{\Pper\left(\thetas - \thetat\right)}$ by decomposing $\left(\thetat - \thetas\right)$ to its top-$k$ span component and bottom-$k$ span component. First notice that since $\Pper = \tilde{\vV}_{>n} \tilde{\vV}_{>n}^T$, $\norm{\Pper\left(\thetas - \thetat\right)} = \norm{\tilde{\vV}_{>n}^\top\left(\thetas - \thetat\right)}$, we can write $\forall k \in [d]$:
\begin{align}
    \norm{\Pper\left(\thetas - \thetat\right)} &= 
    \|\tilde{\vV}_{>n}^T (\thetat - \vtheta_0)\|  \nonumber\\
    &= \|\tilde{\vV}_{>n}^T \vV \vV^T (\thetat - \vtheta_0)\|  \nonumber\\
    &= \|\tilde{\vV}_{>n}^T \vV_{\le k} \vV_{\le k}^T(\thetat - \vtheta_0) + \tilde{\vV}_{>n}^T \vV_{> k} \vV_{> k}^T (\thetat - \vtheta_0)\|  \nonumber\\
    &\le \|\tilde{\vV}_{>n}^T \vV_{\le k}\| \|\vV_{\le k}^T(\thetat - \vtheta_0)\|
    + \|\tilde{\vV}_{>n}^T \vV_{> k}\| \|\vV_{> k}^T (\thetat - \vtheta_0)\|, \label{eq:theta_bar_proj}
\end{align}
Where the last inequality is due to Cauchy Schwartz for matrix-vector.
The last step in the proof is to bound $\|\tilde{\vV}_{>n}^T \vV_{\le k}\|$ by using \cref{lem:davis-kahan} $\forall k \in [d]: \lambda_k > 0$, and bound $\|\tilde{\vV}_{>n}^T \vV_{> k}\|$ by 1 as follows:
\begin{align*}
    \|\tilde{\vV}_{>n}^T \vV_{> k}\| \le \|\tilde{\vV}_{>n}\| \|\vV_{> k}\| \le 1,
\end{align*}
due to $\tilde{\vV}_{>n}$ and $\vV_{> k}$ being orthonormal matrices and because spectral norm is sub-multiplicative.

Plugging \eqref{eq:theta_bar_proj} into \eqref{eq:linear_upper_cauchy} gives the inequality:
\begin{align*}
    R(\valpha) \le \norm{\frac{\norm{\vSigma - \tilde{\vSigma}}^{\nicefrac{3}{2}}}{\lambda_k}\norm{\mat{P}_{\le k}\left(\thetas - \thetat\right)} + \norm{\vSigma - \tilde{\vSigma}}^{\nicefrac{1}{2}}\norm{\mat{P}_{> k}\left(\thetas - \thetat\right)}}^2.
\end{align*}
Since $2a^2 + 2b^2 \ge (a+b)^2$, it follows that:
\begin{align*}
    R(\valpha) \le \frac{2\norm{\vSigma - \tilde{\vSigma}}^{3}}{\lambda_k^2}\norm{\mat{P}_{\le k}\left(\thetas - \thetat\right)}^2 + 2\norm{\vSigma - \tilde{\vSigma}}\norm{\mat{P}_{> k}\left(\thetas - \thetat\right)}^2.
\end{align*}

To conclude the proof we apply \cref{thm:concentration} from \cite{koltchinskii2017concentration} to provide a high probability bound for $\norm{\vSigma - \tilde{\vSigma}}$, as was done in \cite{bartlett2020benign}.
\end{proof}

%% file: proof_deep_linear.tex
\section{Proofs for deep linear networks}\label{appendix:deep_linear}

In this section we analyze the solution achieved by applying gradient flow optimization to fine-tuning a deep linear regression task (i.e. a regression task using a deep linear network as the regression model).

Our results show that the population risk of a fine-tuned deep linear model depends not only on the source and target tasks and the target covariance, as was shown in the previous section, but also on the depth of the model. We show that as the depth of the model goes to infinity, its population risk depends on the difference between the directions of the source and target task (i.e. the difference between their normalized vectors), instead on the difference between the un-normalized task vectors.

In \cref{sec:depth_effect} this is shown by analysing two settings where this effect is most pronounced: one where we make an assumption on the target task (but not on the target covariance), and one where we make an assumption on the target covariance (but not on the target task).

We conclude in \cref{sec:deep_freeze} by showing that fine-tuning only some of the layers can lead to failure to learn.

We begin by recalling some definitions. An $L$-layer linear fully-connected network is defined as
\begin{align*}
    \vbeta(t) = \vW_1(t) \cdots \vW_{L-1}(t) \vW_L(t),
\end{align*}
where $\vW_l \in \R^{d_{l} \times d_{l+1}}$ for $l \in [L-1]$ (we use $d_1 = d$) and $\vW_L \in \R^{d_{L}}$. Thus, the linear network is equivalent to a linear function with weights $\vbeta$.

The weights of a deep linear network are called 0-balanced (or perfectly balanced) at time $t$ if:
\begin{align}
    \vW_j^\top(t) \vW_j(t) = \vW_{j+1}(t) \vW_{j+1}^\top(t) \quad \text{for} \quad j \in [L-1] \label{eq:balancedness_supp}.
\end{align}

\subsection{Proof of Theorem 5.2: The inductive bias of deep linear network fine-tuning}

For this section, let $\vect{u}_l$, $\vect{v}_l$ and $s_l$ denote the top left singular vector, top right singular vector and top singular value of the weights $\vW_l$, respectively. Define $t=0$ as the end of pretraining.


Before proving the theorem, we state several useful lemmas.
\begin{lemma} \label{lm:balancedness_top_singular}
Assume that at time $t$ the weights $\vW_{1}(t), \ldots, \vW_{L}(t)$ are 0-balanced. Then $\vW_{l}(t) = \vu_{l}(t) s_{l}(t) \vv_{l}^\top(t)$,
    \begin{align} \label{eq:singularvectors_alignment} 
        \vect{v}_l(t) = \vect{u}_{l+1}(t),
    \end{align}
    and:
    \begin{align} \label{eq:singularvalues_equivalence} 
    s_l(t) = \|\vbeta(t)\|^{\nicefrac{1}{L}} \text { for } l \in [L].
\end{align}
\end{lemma}
\begin{proof}[Proof for \cref{lm:balancedness_top_singular}]
This proof is a similar to the proof of Theorem 1 in \cite{arora2018optimization}.
Focusing on $j = L-1$ balancedness implies that:
\begin{align*}
    \vW_{L-1}(t)^\top \vW_{L-1}(t) = \vW_L(t) \vW_{L}(t)^\top.
\end{align*}
Hence, $\vW_{L-1}^\top(t) \vW_{L-1}(t)$ is (at most) rank-1 and so is $\vW_{L-1}(t)$. By iterating $j$ from $L-2$ to $1$, it follows  that $\vW_l(t)$ is rank-1 for $j \in [L]$.

Consider the SVD of the weights at time $t$. Since all weights are rank-1, they can be decomposed such that
\begin{align*}
    \vW_l(t) = \vect{u}_l(t) s_l(t) \vect{v}_l(t)^\top.
\end{align*}
Plugging this into \eqref{eq:balancedness_supp} it follows that
\begin{align*}
    \vect{v}_{j}(t) s_j^2(t) \vect{v}_j^\top(t) = \vect{u}_{j+1}(t) s_{j+1}^2(t) \vect{u}_{j+1}^\top(t) \quad \text{for} \quad j \in [L-1],
\end{align*}
Thus proving \eqref{eq:singularvectors_alignment} and showing that the top singular values of all the layers in time $t$ are equal to each other.\footnote{maybe add in footnote that because the two matrices have the same SVD, their spectra are equal.}\\
We now consider the norm of the end to end solution at time $t$, $\vbeta(t)$:
\begin{align*}
    \|\vbeta(t)\| &= \|\vW_1(t) \cdots \vW_L(t)\|\\
    &= \|\vect{u}_1(t) s_1(t) \vect{v}_1^\top s_2(t) \cdots s_L(t)\|\\
    &= \|\vect{u}_1(t) \prod_{i=1}^{L}{s_l(t)}\| = \prod_{i=1}^{L}{s_l(t)} \|\vect{u}_1(t)\|  = \prod_{i=1}^{L}{s_l(t)}.
\end{align*}
Since all of the top singular values at time $t$ equal each other, and $\norm{u_1} = 1$ by construction, the result follows.
\end{proof}

The following Lemma is also used in the analysis:
\begin{lemma}[Theorem 1 from \cite{arora2018optimization}] \label{lm:balanced_training}
Suppose a deep linear network is optimized using GF, starting from a 0-balanced initialization, i.e. initialization in which weights are 0-balanced. Then the weights stay balanced throughout optimization.
\end{lemma}

We are now ready to prove the theorem.
\begin{proof}[Proof of Theorem 5.2]
First consider the pretraining of the model under \cref{asm:pretrain_supp}. Assume that before the pretraining, the model weights are perfectly balanced. From \cref{lm:balanced_training} it follows that after pretraining on the source task, i.e. at $t=0$, the weights of the model are still balanced. From \cref{lm:balancedness_top_singular}, this means they are also rank-1. From \cref{asm:pretrain_supp}:
\begin{align*}
    \vX_S \beta(0) = \vy_S,
\end{align*}
and since $n_S > d$ this implies:
\begin{align} \label{eq:deep_linear_pt_sol}
    \vbeta(0) = \thetas.
\end{align}
\cref{lm:balancedness_top_singular} gives us that:
\begin{align*}
    \vbeta(0) = \vW_1(0) \cdots \vW_L(0) = \vu_1(0) \prod_{i=1}^{L}{s_l(0)} = \vu_1(0) s_1^L(0),
\end{align*}
Hence:
\begin{align*}
    \vect{u}_1(0) =\frac{\thetas}{\|\thetas\|},
\end{align*}
and
\begin{align} \label{eq:pt_singularvalues} 
    s_1(0) = \|\thetas\|^{\nicefrac{1}{L}},
\end{align}
Hence:
\begin{align} \label{eq:W_1_pretrain}
    \vW_1(0) = \vect{u}_1(0) s_{1}(0) \vect{v}_1^\top(0) = \frac{\thetas}{\|\thetas\|} \|\thetas\|^{\nicefrac{1}{L}} \vect{v}_1^\top(0)= \frac{\thetas}{\|\thetas\|^{\nicefrac{(L-1)}{L}}} \vect{v}_1^\top(0).
\end{align}

We next analyze the fine-tuning dynamics. \cref{lm:balanced_training} ensures that if the pretrained model has 0-balanced weights, then the weights will remain 0-balanced during finetune. This implies that \cref{lm:balancedness_top_singular} holds for all $t \ge 0$.

Observe the gradient flow dynamics of the layers during fine-tuning:
\begin{align*}
    \dot{\vW_l}(t) &= - \vW_{l-1}^T(t) \cdots \vW_1^T(t) \vX^T \vect{r}(t) \vW_L^T(t) \cdots \vW_{l+1}^T(t) \text { for } l \in [L],
\end{align*}
where $\vect{r}(t) \in \R^n$ is the residual vector satisfying $[\vect{r}]_i = \vx_i^\top\vbeta(t) - \vy_{i}$.

From \cref{lm:balancedness_top_singular}:
\begin{align*}
    \dot{\vW_l}(t) =& - \vect{v}_{l-1}(t)s_{l-1}(t)\vect{u}_{l-1}^T(t)\vect{v}_{l-2}(t)s_{l-2}(t)\vect{u}_{l-2}^T(t) \cdots\\
    &\vect{v}_{1}(t)s_{1}(t)\vect{u}_{1}^T(t) \vX^T \vect{r}(t) \vect{v}_{L}(t)s_{L-1}(t)\vect{u}_{L}^T(t) \cdots\\ &\vect{v}_{l+1}(t)s_{l+1}(t)\vect{u}_{l+1}^T(t) \text { for } l \in [L].
\end{align*}
Using \eqref{eq:singularvectors_alignment} and \eqref{eq:singularvalues_equivalence} it follows that $\forall t\ge 0$:
\begin{align*}
    \dot{\vW_l}(t) &= - \vect{v}_{l-1}(t)\left(\prod_{i=1}^{l-1}{s_i}(t)\right)\vect{u}_{1}(t)^T \vX^T \vect{r}(t) \left(\prod_{i=l+1}^{L}{s_i}(t)\right)\vect{u}_{l+1}^T(t) \text { for } l \in [L]\\
    &= - \vect{v}_{l-1}(t) s^{l-1}(t)\vect{u}_{1}^T(t) \vX^T \vect{r}(t) s^{L-l}(t)\vect{u}_{l+1}^T(t) \text { for } l \in [L].
\end{align*}
For $\vW_1$,
\begin{align} 
    \dot{\vW_1}(t) &= -\vX^T\vect{r}(t) s^{L-1}(t)\vect{u}_{2}^T(t)
    = -\vX^T\vect{r}(t) s^{L-1}(t) \vect{v}_{1}^T(t), \label{eq:W_1_grad}
\end{align}
Where the last equality is due to \eqref{eq:singularvectors_alignment}.
Hence $\dot{\vW_1}$ is always a rank-1 matrix whose columns are in the row space of $\vX$. This implies that the decomposition $\vW_1$ into two orthogonal components $\vW_1^\perp$ and $\vW_1^\parallel$ so that $\vW_1^\parallel = \Ppar \vW_1$ and $\vW_1^\perp = \Pper \vW_1$ yields that $\forall t \ge 0$ it follows that
\begin{align*}
    &\dot{\vW}_1^\perp(t) = \mat{0},\\
    &\dot{\vW}_1^\parallel(t) = \dot{\vW}_1(t) = \vX^T\vect{r}(t) s^{L-1}(t) \vect{v}_{1}^T(t).
\end{align*}
Hence, $\vW_1^\perp(t)$ does not change for all $t \geq 0$. Using \eqref{eq:W_1_pretrain} it follows:
\begin{align}
    \vW_1^\perp(t) &= \vW_1^\perp(0) \label{eq:W_1_perp}\\ &=\Pper\left(\frac{\thetas}{\|\thetas\|^{\frac{L-1}{L}}} \vect{v}_1^\top(0)\right)\nonumber\\
    &=\frac{\Pper\thetas}{\|\thetas\|^{\frac{L-1}{L}}} \vect{v}_1^\top(0).\label{eq:W_1_perp_0}
\end{align}

The next lemma states that $\vect{v}_1(t)$ does not change during optimization if $\norm{\Pper\vW_1(0)}_F > 0$.
\begin{lemma} \label{lm:v_1_finegrained}
Suppose we run GF over a deep linear network starting from 0-balanced initialization. Also assume that at initialization $\vW_1(0)$ is rank-1 and:
\begin{align*}
    \norm{\Pper\vW_1(0)}_F > 0,
\end{align*}
Then for all $t >0$:
\begin{align*}
    \vect{v}_1(t) = \vect{v}_1(0).
\end{align*}
\end{lemma}
\begin{proof}
Assume towards contradiction that there exists $t > 0$ s.t. $\vect{v}_1(t) \neq \vect{v}_1(0)$.\\
 From $\vW_1(t)$ being rank-1 (\cref{lm:balancedness_top_singular}), it follows that
\begin{align*}
    \Pper \vW_1(t) = \Pper \vect{u}_1(t) s(t) \vect{v}_1^\top(t) = \left(\Pper \vect{u}_1(t) s(t)\right) \vect{v}_1^\top(t),
\end{align*}
And from the decomposition of $\vW_1(t)$ to $\vW_1^{\parallel}(t)$ and $\vW_1^{\perp}(t)$, \eqref{eq:W_1_perp} and $\vW_1(0)$ being rank-1 it follows that:
\begin{align*}
    \Pper\vW_1(t) &= \vW_1^{\perp}(t) = \vW_1^{\perp}(0) =\Pper\vect{u}_1(0) s_1(0) \vect{v}_1^\top(0),
\end{align*}
Hence:
\begin{align*}
    \left(\Pper \vect{u}_1(t) s(t)\right) \vect{v}_1^\top(t) = \left(\Pper\vect{u}_1(0) s_1(0)\right) \vect{v}_1^\top(0).
\end{align*}
From \eqref{eq:W_1_grad} we see that the orthogonal part of $\vu_1(t)$ does not change during fine-tune:
\begin{align*}
    \dot{\vu_1}(t) = \dot{\vW_1}(t) \cdot \frac{\partial \vW_1(t)}{\partial \vu_1(t)} = -\vX^T\vect{r}(t) s^{L-1}(t) \vect{v}_{1}^T(t) \vect{v}_{1}(t) s(t) = -\vX^T \vect{r}(t) s^L(t)
\end{align*}
hence:
\begin{align}
    \Pper \dot{\vu_1}(t) = 0 \Rightarrow \Pper \vect{u}_1(t) = \Pper\vect{u}_1(0) \label{eq:perp_u1}. 
\end{align}
Since $\vect{v}_1(t) \neq \vect{v}_1(0)$, and because non-degenerate singular values always have unique left and right singular vectors (up to a sign), $\vW_1^{\perp}(t) = \vW_1^{\perp}(0)$ only if:
\begin{align*}
    s(t) = s_1(0) = 0,
\end{align*}
by contradiction to the assumption that $s_1(0) = \norm{\Pper\vW_1(0)}_F > 0$, 
or if $\vect{v}_1(t) = - \vect{v}_1(0)$ and $\Pper \vect{u}_1(t) = -\Pper\vect{u}_1(0)$, which contradicts \eqref{eq:perp_u1}.
\end{proof}

In the case where $\norm{\Pper\vW_1(0)}_F = 0$, since $\Pper\vW_1(t) = \Pper\vW_1(0)$, it follows that $\vW_1(t) = \Ppar\vW_1(t)$, which is similar to the case in \cite{yun2020unifying}, for which the solution is known to be $\Ppar \thetat$. Also, from \eqref{eq:W_1_perp_0}, this implies $\Pper \thetas = 0$, and the expression for the end-to-end solution in Eq.5 in the main text holds.

The analysis continues for $\norm{\Pper\vW_1(0)}_F > 0$. By using \cref{lm:balancedness_top_singular} and \cref{lm:v_1_finegrained} it follows that:
\begin{align}
    \vW_1^{\perp}(t) \vW_2(t) \cdots \vW_L(t) &\stackrel{(1)}{=} \vW_1^{\perp}(0) \vW_2(t) \cdots \vW_L(t)  \nonumber\\
    &\stackrel{(2)}{=} \frac{\Pper\thetas}{\|\thetas\|^{\frac{L-1}{L}}} \vect{v}_1^\top(0) \vW_2(t) \cdots \vW_L(t)  \nonumber\\
    &\stackrel{(3)}{=} \frac{\Pper\thetas}{\|\thetas\|^{\frac{L-1}{L}}} \vect{v}_1^\top(t) \vW_2(t) \cdots \vW_L(t)  \nonumber\\
    &=\frac{\Pper\thetas}{\|\thetas\|^{\frac{L-1}{L}}} \vect{v}_1^\top(t) \vect{u}_2(t) \|\vbeta(t)\|^{\frac{L-1}{L}}  \nonumber\\
    &\stackrel{(4)}{=}\frac{\Pper\thetas}{\|\thetas\|^{\frac{L-1}{L}}} \vect{v}_1^\top(t) \vect{v}_1(t) \|\vbeta(t)\|^{\frac{L-1}{L}}  \nonumber\\
    &= \left(\frac{\|\vbeta(t)\|}{\|\thetas\|}\right)^{\frac{L-1}{L}}\Pper\thetas \label{eq:W_1_orth_sol}.
\end{align}
With (1) due to \eqref{eq:W_1_perp}, (2) due to \eqref{eq:W_1_perp_0}, (3) due to \cref{lm:v_1_finegrained} and (4) due to \cref{lm:balancedness_top_singular}.
From the requirement of \cref{asm:train_loss_supp} that $\lim_{t \to \infty} \vX \vbeta(t) = \vy$, it follows that:
\begin{align}
    & \lim_{t \to \infty}\vX \vW_1(t) \cdots \vW_L(t) = \vy \nonumber \\
    \Rightarrow & \lim_{t \to \infty}\vX \vW_1^\parallel(t) \cdot \vW_2(t) \cdots \vW_L(t) = \vy \nonumber \\
    \Rightarrow & \lim_{t \to \infty}\vW_1^\parallel(t) \cdot \vW_2(t) \cdots \vW_L(t) = \vX^T \left(\vX\vX^T\right)^{-1}\vy, \label{eq:W_1_proj_sol}
\end{align}
Which is the only solution for this equation in the span of $\vX$, and  due to \cref{asm:non-singular_supp}.\\
Eq.5 in the main text follows from \eqref{eq:W_1_orth_sol} and \eqref{eq:W_1_proj_sol}:

\begin{align}
    \lim_{t \to \infty}\vbeta(t) &= \lim_{t \to \infty} \vW_1(t) \cdot \vW_2(t) \cdots \vW_L(t)  \nonumber\\
    &=\lim_{t \to \infty} \left(\vW_1^{\parallel}(t) + \vW_1^{\perp}(t)\right) \cdot \vW_2(t) \cdots \vW_L(t)  \nonumber\\
    &=\lim_{t \to \infty} \vW_1^{\perp}(t)\cdot \vW_2(t) \cdots \vW_L(t) + \vW_1^{\parallel}(t)\cdot \vW_2(t) \cdots \vW_L(t)  \nonumber\\
    &= \left(\frac{\|\lim_{t \to \infty}\vbeta(t)\|}{\|\thetas\|}\right)^{\frac{L-1}{L}}\Pper\thetas + \Ppar\thetat \label{eq:deep_linear_solution}.
\end{align}

To prove Eq.6 from the main text, consider the norm of $\lim_{t \to \infty}\vbeta(t)$.

\begin{align*}
    &\|\lim_{t \to \infty}\vbeta(t)\| =  \sqrt{\left(\frac{\|\lim_{t \to \infty}\vbeta(t)\|}{\|\thetas\|}\right)^{\frac{2(L-1)}{L}}\|\Pper\thetas\|^2 + \|\Ppar\thetat\|^2} \\
    \Rightarrow &\|\lim_{t \to \infty}\vbeta(t)\|^2 = \left(\frac{\|\lim_{t \to \infty}\vbeta(t)\|}{\|\thetas\|}\right)^{\frac{2(L-1)}{L}}\|\Pper\thetas\|^2 + \|\Ppar\thetat\|^2 \\
    \Rightarrow & \|\lim_{t \to \infty}\vbeta(t)\|^2 - \left(\frac{\|\lim_{t \to \infty}\vbeta(t)\|}{\|\thetas\|}\right)^{\frac{2(L-1)}{L}}\|\Pper\thetas\|^2 - \|\Ppar\thetat\|^2 = 0.
\end{align*}
At the limit $L \rightarrow \infty$ we get:
\begin{align*}
    &\lim_{l \to \infty}\left(\|\lim_{t \to \infty}\vbeta(t)\|^2 - \left(\frac{\|\lim_{t \to \infty}\vbeta(t)\|}{\|\thetas\|}\right)^{\frac{2(L-1)}{L}}\|\Pper\thetas\|^2 - \|\Ppar\thetat\|^2\right)\\
    = &\|\lim_{l \to \infty}\lim_{t \to \infty}\vbeta(t)\|^2 - \left(\frac{\|\lim_{l \to \infty}\lim_{t \to \infty}\vbeta(t)\|}{\|\thetas\|}\right)^2\|\Pper\thetas\|^2 - \|\Ppar\thetat\|^2 = 0 \\
    \Rightarrow&\frac{\|\lim_{l \to \infty}\lim_{t \to \infty}\vbeta(t)\|^2}{\norm{\thetas}^2}\left(\norm{\thetas}^2 - \|\Pper\thetas\|^2\right) = \|\Ppar\thetat\|^2,
\end{align*}
Thus:
\begin{align*}
    \frac{\|\lim_{l \to \infty}\lim_{t \to \infty}\vbeta(t)\|}{\|\thetas\|} &= \frac{\|\Ppar\thetat\|}{\sqrt{\|\thetas\|^2 - \|\Pper\thetas\|^2}}
    = \frac{\|\Ppar\thetat\|}{\|\Ppar\thetas\|}.
\end{align*}
And it follows  that at this limit:
\begin{align}
    \lim_{L \to \infty}\lim_{t \to \infty}\vbeta(t) &= \frac{\|\Ppar\thetat\|}{\|\Ppar\thetas\|}\Pper\thetas + \Ppar\thetat. \label{eq:deep_linear_infty_solution}
\end{align}
\end{proof}

From the same lines of proof as in \cref{subsec:proof_41} it follows that
\begin{corollary}\label{cor:deep_infty_risk}
    For the conditions in Theorem 5.2 in the main text, 
\begin{align*}
    R(\lim_{L \to \infty}\lim_{t \to \infty}\vbeta(t)) = \norm{\vSigma^{0.5}\left(\Pper(\thetat - \frac{\|\Ppar\thetat\|}{\|\Ppar\thetas\|}\thetas)\right)}^2.
\end{align*}
\end{corollary}

\subsection{Proofs of Theorems 5.3 and 5.4: How does depth affect the population risk?}\label{sec:depth_effect}

\cref{cor:deep_infty_risk} above contains dependence on $\Ppar$ which is a random variable. We next provide high-probability risk bounds that can be derived from this result. The bounds are obtained under slightly different assumptions, either on the target task or on the target distribution, but both highlight the fact that fine-tuning in the $L\to\infty$ case will depend on $\hat{\vtheta}_S-\hat{\vtheta}_T$ rather than the un-normalized $\thetas - \thetat$.



Recall the definition of the fine-tuning solution as $L\to\infty$:
\begin{align*}
    \vbeta \triangleq \lim_{L \to \infty}\lim_{t \to \infty} \vbeta(t).
\end{align*}

In the first setting we will assume that $\thetat$ is a scaled version of $\thetas$, without any assumptions on $\dist$. Theorem 5.3 from the main text demonstrates a gap between perfect fine-tuning for the $L\to\infty$ case and non-zero fine-tuning error for $L=1$.
\begin{proof}[Proof of Theorem 5.3 (Main Text)]

First notice:
\begin{align}
    \frac{\|\Ppar\thetat\|}{\|\Ppar\thetas\|}
    &= \frac{\|\Ppar \alpha\thetas\|}{\|\Ppar\thetas\|}= \alpha\frac{\|\Ppar\thetas\|}{\|\Ppar\thetas\|} = \alpha,
\end{align}
which from Eq.6 in the main text gives the solution
\begin{align*}
    \vbeta &= \alpha\Pper\thetas + \Ppar\thetat = \Pper\thetat + \Ppar \thetat = \thetat.
\end{align*}
On the other hand, for the $L=1$ solution $\valpha$ it follows from Eq.2 in the main text that
\begin{align*}
    \norm{\vSigma^{0.5}\Pper\left(\thetat - \thetas\right)}^2 &= \norm{\vSigma^{0.5}\Pper\left(\thetat - \frac{\thetat}{\alpha}\right)}^2 \\
    & = \left(\frac{\alpha - 1}{\alpha}\right)^2\norm{\vSigma^{0.5}\Pper\thetat}^2,
\end{align*}
which is greater than zero for all $\alpha \neq 1$.
\end{proof}

In the second setting we assume that $\dist = \calN(0,1)^{d}$, without any assumptions on $\thetat$. Here it shows that while the population risk of the $L=1$ solution depends on $\norm{\thetat - \thetas}$, the population risk of the infinitely-deep linear solution depends on the normalized $\norm{\thetath - \thetash}$ and $\norm{\thetat}$, i.e. on the alignment of $\thetat$ and $\thetas$ and the norm of $\thetat$.
\begin{customthm}{5.4}[Main Text]
    Assume that the conditions of Theorem 5.2 hold, and let $\vX \sim \calN(0,1)^{d}$. Suppose $n\le d$, then there exists a constant $c>0$ such that for an $\epsilon > 0$ it holds that with probability at least $1 - 4\exp(-c\epsilon^2 n) -4\exp\left(-c\epsilon^2 (d-n)\right)$ the population risk for the $L \to \infty$ end-to-end $\vbeta$ is bounded:
    \begin{align}
    R(\vbeta) 
    &\le \frac{d-n}{d}(1+\epsilon)^2\norm{\thetat}^2\norm{\thetath - \thetash}^2 + \frac{d-n}{d}\zeta(\norm{\thetat})^2,
\end{align}
    for $\zeta(\norm{\thetat}) \approx \epsilon\norm{\thetat}$.
    For the $L=1$ linear regression solution $\valpha$ this risk is bounded by
    \begin{align*}
        R(\valpha) \le \frac{d-n}{d}(1+\epsilon)^2\norm{\thetat - \thetas}^2.
    \end{align*}
\end{customthm}
\begin{proof}[Proof of Theorem 5.5 (Main Text)]

We start by analyzing $R(\vbeta)$:
\begin{align*}
    R(\vbeta) &=
        \norm{\vSigma^{0.5}\Pper\left(\thetat - \frac{\norm{\Ppar\thetat}}{\norm{\Ppar\thetas}}\thetas\right)}^2 \\ &\stackrel{(1)}{=} \norm{\mat{I}^{0.5}\Pper\left(\thetat - \frac{\norm{\Ppar\thetat}}{\norm{\Ppar\thetas}}\thetas\right)}^2\\
    &= \norm{\Pper\left(\thetat -
    \frac{\norm{\Pper\thetat}}{\norm{\Ppar\thetas}}\thetas\right)}^2,
\end{align*}
where (1) is due to $\vSigma = \mat{I}$ from the definition of the distribution of $\vX$.
We then bound the RHS with:
\begin{align*}
    &\norm{\Pper\left(\thetat - \frac{\norm{\Ppar\thetat}}{\norm{\Ppar\thetas}}\thetas\right)}^2\\
    &\le \norm{\Pper\left(\thetat - \frac{\norm{\Ppar\thetat}}{\norm{\Ppar\thetas}}\thetas\right) - \Pper\left(\norm{\thetat}(\thetath - \thetash)\right) +  \Pper\left(\norm{\thetat}(\thetath - \thetash)\right)}^2\\
    &\le  \norm{\Pper\left(\thetat - \frac{\norm{\Ppar\thetat}}{\norm{\Ppar\thetas}}\thetas\right) - \Pper\left(\norm{\thetat}(\thetath - \thetash)\right)}^2 + \norm{\Pper\left(\norm{\thetat}(\thetath - \thetash)\right)}^2.
\end{align*}

We see that we can bound the expression on the left:
\begin{align*}
    &\norm{\Pper\left(\thetat - \frac{\norm{\Ppar\thetat}}{\norm{\Ppar\thetas}}\thetas\right) - \Pper\left(\norm{\thetat}(\thetath - \thetash)\right)}^2 \\
    &= \norm{\Pper\left(\thetat - \frac{\norm{\Ppar\thetat}}{\norm{\Ppar\thetas}}\thetas - \thetat + \norm{\thetat}\thetash\right)}^2\\
    &= \norm{\Pper\left(\frac{\norm{\thetat}}{\norm{\thetas}}\thetas -  \frac{\norm{\Ppar\thetat}}{\norm{\Ppar\thetas}}\thetas\right)}^2 \\
    &\le \norm{\Pper\thetas\left(\frac{\norm{\thetat}}{\norm{\thetas}} - \frac{\norm{\Ppar\thetat}}{\norm{\Ppar\thetas}}\right)}^2 \\
    &\le \norm{\Pper\thetas}^2\norm{\frac{\norm{\thetat}}{\norm{\thetas}} - \frac{\norm{\Ppar\thetat}}{\norm{\Ppar\thetas}}}^2 
\end{align*}

Let $\Ppar$ be the projection matrix onto the row space of $\vX$, then from \cite{393720}, $\Ppar$ is a projection onto a random $n$-dimensional subspace uniformly distributed in the Grassmannian $\mat{G}_{d, n}$, and $\Pper$ is a projection onto a random $d-n$-dimensional subspace uniformly distributed in the Grassmannian $\mat{G}_{d, d-n}$. 

According to Lemma 5.3.2 in \cite{vershynin2018high}, with probability at least $1 - 4\exp(-c\epsilon^2 n)$
\begin{align*}
    \frac{1-\epsilon}{1+\epsilon}\frac{\norm{\thetat}}{\norm{\thetas}} \le
    \frac{\|\Ppar\thetat\|}{\|\Ppar\thetas\|} \le \frac{1+\epsilon}{1-\epsilon}\frac{\norm{\thetat}}{\norm{\thetas}},
\end{align*}
which bounds:
\begin{align*}
    \norm{\frac{\norm{\thetat}}{\norm{\thetas}} - \frac{\norm{\Ppar\thetat}}{\norm{\Ppar\thetas}}}^2
    &\le \norm{\frac{\norm{\thetat}}{\norm{\thetas}} - \frac{1+\epsilon}{1-\epsilon}\frac{\norm{\thetat}}{\norm{\thetas}}}^2\\
    &=
    \left(\frac{\norm{\thetat}}{\norm{\thetas}}\right)^2\frac{4\epsilon^2}{(1-\epsilon)^2}.
\end{align*}
Again, by applying Lemma 5.3.2 from \cite{vershynin2018high}, with probability at least $1 - 4\exp\left(-c\epsilon^2(d-n)\right) - 2\exp\left(-c\epsilon^2(d-n)\right):$
\begin{align*}
    &\norm{\Pper\thetas}^2 \le (1+\epsilon)^2\frac{d-n}{d}\norm{\thetas}^2,\\
    &\norm{\Pper\norm{\thetat}\left(\thetath - \thetash\right)}^2 \le (1+\epsilon)^2\frac{d-n}{d}\norm{\norm{\thetat}\left(\thetath - \thetash\right)}^2.
\end{align*}
Thus the following bound is obtained:
\begin{align*}
    R(\vbeta) &\le \norm{\Pper\left(\thetat - \frac{\norm{\Ppar\thetat}}{\norm{\Ppar\thetas}}\thetas\right) - \Pper\left(\norm{\thetat}(\thetath - \thetash)\right)}^2 + \norm{\Pper\left(\norm{\thetat}(\thetath - \thetash)\right)}^2 \\
    &\le (1+\epsilon)^2\frac{d-n}{d}\norm{\norm{\thetat}(\thetath -\thetash)}^2 + \frac{4\epsilon^2(1+\epsilon)^2}{(1-\epsilon)^2}\frac{d-n}{d}\norm{\thetas}^2\frac{\norm{\thetat}^2}{\norm{\thetas}^2} \\
    &= (1+\epsilon)^2\frac{d-n}{d}\norm{\norm{\thetat}(\thetath -\thetash)}^2 + \frac{4\epsilon^2(1+\epsilon)^2}{(1-\epsilon)^2}\frac{d-n}{d}\norm{\thetat}^2.
\end{align*}
Define $\zeta(\norm{\thetat}) = \frac{2\epsilon(1+\epsilon)}{(1-\epsilon)}\norm{\thetat}$, which concludes the proof for the infinite depth case.

Now for the upper bound of the population risk of the $L=1$ solution $\valpha$. Look at Eq.2, and from $\Pper$ being a random projection, it follows that with probability at least $1 - 2\exp\left(-c\epsilon^2(d-n)\right)$:

\begin{align*}
    R(\valpha) &\le \norm{\vSigma^{0.5}\Pper\left(\thetat - \thetas\right)}^2\\
    &= \norm{\vI\Pper\left(\thetat - \thetas\right)}^2\\
    &\le (1+\epsilon)^2\frac{d-n}{d}\norm{\thetat-\thetas}^2.
\end{align*}
\end{proof}


\subsection{Proof of Theorem 5.5: The effect of fixing layers during fine-tuning} \label{sec:deep_freeze}

\begin{proof} 
Since we assume that the weights before pretraining are 0-balanced, it follows from \cref{lm:balancedness_top_singular} and \cref{lm:balanced_training} that all layers $\vW_1(t), \ldots \vW_k(t)$ are rank-1. From \cref{asm:pretrain_supp} it follows that at the end of pretraining $\vbeta(0) = \thetas$, and from \eqref{eq:W_1_pretrain} it follows that $u_1(0) = \thetash$.

Consider the setting where the first $k$ layers are fixed. It follows that
\begin{align*}
    \vW_i(t) = \vW_i(0) \quad \forall t \ge 0 \text{,} \quad 0 \le i \le k.
\end{align*}
Then from \cref{lm:balancedness_top_singular} it follows  that for $t \ge 0$ and for any $\vx \in \R^d$:
\begin{align*}
    \vx^\top \vW_1(t) \cdots \vW_k(t) &= \vx^\top \vW_1(0) \cdots \vW_k(0) 
    = \vx^\top \vu_1(0) \prod_{i=1}^{k}{s_i} \vv_k^\top(0)\\ &= \vx^\top \norm{\thetas}^{\nicefrac{k}{L}} \vu_1(0) \norm{\thetas}^{\nicefrac{k}{L}} \vv_k^\top(0) \\
    &= \vx^\top \thetas \norm{\thetas}^{\nicefrac{k-L}{L}} \vv_k^\top(0) = \norm{\thetas}^{\nicefrac{k-L}{L}} \langle \vx, \thetas \rangle \vv_k^\top(0).
\end{align*}
Let's define
\begin{align*}
    \vect{b}(t) \triangleq \vW_{k+1}(t) \cdots \vW_L(t),
\end{align*}
then for any constant $c_1(t) \triangleq \langle \vv_k, \vect{b}(t) \rangle$ it follows :
\begin{align*}
    \vx^\top \vbeta(t) &= \vx^\top \vW_1(t) \cdots \vW_k(t) \cdot \vW_{k+1}(t) \cdots \vW_L(t)\\
    &= \norm{\thetas}^{\nicefrac{k-L}{L}} \langle \vx, \thetas \rangle \vv_k^\top(0) \vect{b}(t)\\
    &= c_1(t) \norm{\thetas}^{\nicefrac{k-L}{L}} \langle \vx, \thetas \rangle.
\end{align*}
By setting $c(t) = c_1(t) \norm{\thetas}^{\nicefrac{k-L}{L}}$ we conclude the proof.
\end{proof}

%% file: proof_shallow_finetune.tex
\section{Proofs for the shallow ReLU section}\label{appendix:shallow_finetune}
This section shows that fine-tuning from a shallow ReLU model pretrained on $\thetas$ has sample complexity depending on $\norm{\thetat - \thetas}$, compared to training from a random initialization which depends on $\norm{\thetat}$.


We would like to adapt the results from \cite{arora2019fine} to the case of fine-tuning in the NTK regime, where we can take better advantage of the fact that the bound in Theorem 4.1 in  \cite{arora2019fine} fundamentally depends on $\norm{\vtildey}$, thus enabling us to bound the distance of each weight from $t=0$ by using $\vtildey$ instead of $\vy$ for our case, where $\vu(0)$ is known.

The proof scheme is as follows:
\begin{enumerate}
    \item First we show that $\norm{\vH(t) - \vH^\infty} = O(\tfrac{1}{\sqrt{m}})$, thus ensuring we are indeed in the NTK regime for $m$ bounded from bellow as in Theorem 6.1 from the main text.
    \item Then, we can use an adaption of Theorem 4.1 from \cite{arora2019fine} to bound the distance of each weight $\norm{\vw_r(t) - \vw_r(0)} \forall r \in [m]$.
    \item Since $\vW(0)$ is fixed, we can use the Rademacher bound in Theorem 5.1 from \cite{arora2019fine} with $\vW(0)$ instead of $\vW(\init)$ to obtain a bound that depends on $\vtildey^\top \vH^\infty \vtildey$ instead of $\vy^\top \vH^\infty \vy$.
    \item For $\vtildey = \vX(\thetat - \thetas)$, we can use Corollary 6.2 from \cite{arora2019fine} with $\vbeta = (\thetat - \thetas)$ to obtain the generalization error using the Rademacher bound above.
\end{enumerate}

\subsection{Staying in the NTK regime}
Start with the first item: showing that $\norm{\vH(t) - \vH^\infty} = O(\tfrac{1}{\sqrt{m}})$. This is done by bounding the distance each $\vw_r \forall r \in [m]$ travels during both the pretraining and fine-tuning optimization, which is achievable by using Theorem 4.1 from \cite{du2018gradient} "as is" for the pretraining part, and adapting it to the fine-tuning part.

\paragraph{Assumptions}
For brevity, we assume for the pretraining data that $|\vx_{S_i}|\le 1, |y_{S_i}|\le 1$ for all $i \in [n_S]$.  Also assume the following for all results:
\begin{assumption}\label{asmp:pretrain_init}
    We assume that $\vW(\init)$, i.e. the weights at $t=\init$, were i.i.d. initialized $\vw_r \sim \calN(\vect{0},\mat{I})$, $a_r \sim \unif\left[\left\{-1,1\right\}\right]$ for $r \in [m]$.
\end{assumption}
Also assume for $\vX, \vX_s$:
\begin{assumption}\label{asmp:main}
Define matrix $\vH^\infty \in \mathbb{R}^{n \times n}$ with
\begin{align*}
    \vH_{ij} ^\infty= \expect_{\vw \sim N(\vect{0},\mat{I})}\left[\vx_i^\top \vx_j\indict\left\{\vw^\top \vx_i \ge 0, \vw^\top \vx_j \ge 0\right\}\right].
\end{align*}
We assume $\lambda_0 \triangleq \lambda_{\min}\left(\vH^{\infty}\right) > 0$, and $\lambda_{0_S} \triangleq \lambda_{\min}\left(\vH_S^{\infty}\right) > 0$ for $\vH_S$ being the NTK gram matrix of the pretraining data $\vX_S$.
\end{assumption}
The assumption that $\lambda_0 > 0$ is justified by combining \cref{asm:non-singular_supp} and Theorem 3.1 from \cite{du2018gradient}.
The assumption that $\lambdas > 0$, which is actually the assumption for \cref{thm:non_parallel}, holds for most real-data data-sets and w.h.p for most real-life distributions, as discussed in \cite{du2018gradient}.

\begin{assumption}\label{asmp:width_kappa_eta}
    We assume that $m = \Omega\left( \frac{n_s^6}{\lambda_{0_s}^4 \kappa^2 \delta^3} \right)$, $\kappa = O\left(\frac{\epsilon\delta}{\sqrt{n_S}} + \frac{\epsilon\delta}{\sqrt{n}}\right)$ and $\eta_T = O\left( \frac{\lambda_0}{n^2}\right)$, $\eta_S = O\left(\frac{\lambdas}{n_S^2} \right)$.
\end{assumption}

We now restate a few results from \cite{du2018gradient} which are applied directly for the part of pretraining:
\begin{theorem}[Theorem 3.1 from \cite{du2018gradient}]\label{thm:non_parallel}
    If for any $i \ne j$, $\vx_i \nparallel \vx_j$ , then $\lambda_0 > 0$.
\end{theorem}

\begin{theorem}[Theorem 3.3 from \cite{du2018gradient} for pretraining]\label{thm:rand_init_loss} Assume \cref{asmp:pretrain_init}, \cref{asmp:main} and \cref{asmp:width_kappa_eta} hold, then with probability at least $1-\delta$ over the random initialization at time $t=\init$, we have:
\begin{align*}
    &\frac{1}{2}\norm{\vy_s - \vu(\init)} = O\left(n_S/\delta\right).
\end{align*}
\end{theorem}

\begin{lemma}[Lemma C.1 from \cite{arora2019fine}]\label{lem:pretrain_distance}
    Assume \cref{asmp:pretrain_init}, \cref{asmp:main} and \cref{asmp:width_kappa_eta} hold, then there exists $C >0$ such that with probability at least $1-\delta$ over the random initialization at time $t=\init$ we have
    \begin{align*}
        \norm{\vw_r(0) - \vw_r(\init)}_2 &\le \frac{4 \sqrt{n_s} \norm{\vy_s - \vu(\init)}}{\sqrt m \lambdas } \quad \forall r\in[m].
    \end{align*}
\end{lemma}

Plugging \cref{thm:rand_init_loss} into \cref{lem:pretrain_distance} we get:
\begin{corollary}\label{cor:exact_pretrain_distance}
   Assume \cref{asmp:pretrain_init}, \cref{asmp:main} and \cref{asmp:width_kappa_eta} hold, then there exists $C >0$ s.t. with probability at least $1 - 2\delta$ over the random initialization at time $t=\init$ we have
   \begin{align*}
        \norm{\vw_r(0) - \vw_r(\init)}_2 
        &\le \frac{C n_S}{\sqrt{m\delta} \lambdas }\quad \forall r\in[m].
    \end{align*}
\end{corollary}

\begin{lemma}[Lemma 3.2 from \cite{du2018gradient}]\label{lem:close_to_init_small_perturbation}
If $\vw_1,\ldots,\vw_m$ at $t=\init$ are i.i.d. generated from $\calN(\vect{0},\mat{I})$, then with probability at least $1-\delta$, the following holds.
For any set of weight vectors $\vw_1, \ldots,\vw_m \in \R^d$ that satisfy for any $r \in [m]$, $\norm{\vw_r(\init)-\vw_r}_2 \le \frac{c\delta\kappa\lambda_0}{n^2}$ for some small positive constants $c$, then the matrix $\vH \in \mathbb{R}^{n \times n}$ defined by 
\begin{align*}
    \vH_{ij} = \frac{1}{m} \vx_i^\top \vx_j \sum_{r=1}^{m}\indict\left\{
    \vw_r^\top \vx_i \ge 0, \vw_r^\top \vx_j \ge 0\right\}
\end{align*}
 satisfies $\norm{\vH-\vH(\init)}_2 < \frac{\lambda_0}{4}$ and $\lambda_{\min}\left(\vH\right) > \frac{\lambda_0}{2}$.
\end{lemma}

We state the following lemmas that is used in the analysis:
\begin{lemma}[Similar to Lemma C.2 from \cite{arora2019fine}]\label{lem:flipped_neurons_bound}

Assume \cref{asmp:pretrain_init} holds. 
For some $R > 0$ we define:
\begin{align}
    \mat{A}_{r,i} \triangleq \left\{|\vx_i^\top \vw_r(\init)| \le R \right\}, \label{eq:def_A_R_i}
\end{align}
then with probability at least $1-\delta$ on the initialization of $\mat{W}(\init)$ we get:
\begin{align*}
    \E[\indict\{\mat{A}_{r,i}\}] \le \frac{2R}{\sqrt{2\pi}\kappa},
\end{align*}
and:
\begin{align*}
    \sum_{i=1}^{n}{\sum_{r=1}^{m}{\indict\{\mat{A}_{r,i}\}}} = O\left(\frac{mnR}{\kappa\delta}\right)
    .
\end{align*}
\end{lemma}
where the expectation is with respect to $\vW(\init)$.
\begin{proof}
Since $\vw_r(\init)$ has the same distribution as $\calN(0, \kappa^2)$ we have
\begin{align*}
    \E[\indict\{\mat{A}_{r,i}\}] &\le \E[\indict\left\{|\vx_i^\top \vw_r(\init)| \le R \right\}]\\
    &= \Pr_{z\sim \calN(0, \kappa^2)} \left[ |z| \le R \right]
	= \int_{-R}^{R} \frac{1}{\sqrt{2\pi}\kappa} e^{-x^2/2\kappa^2} dx\\
	&\le \frac{2R}{\sqrt{2\pi}\kappa}.
\end{align*}
Then we know $\E\left[\sum_{i=1}^{n}{\sum_{r=1}^{m}{\indict\{\mat{A}_{r,i}\}}}\right] \le \frac{2mnR}{\sqrt{2\pi}\kappa}$. Due to Markov, with probability at least $1-\delta$ we have:
\begin{align*}
    \sum_{i=1}^{n}{\sum_{r=1}^{m}{\indict\{\mat{A}_{r,i}\}}} = O\left(\frac{mnR}{\kappa\delta}\right).
\end{align*}
\end{proof}

We now state our equivalent for Theorem 4.1 from \cite{du2018gradient} :
\begin{theorem}[Adaption of Theorem 4.1 from \cite{du2018gradient}]
\label{thm:du2018}
Suppose \cref{asmp:pretrain_init} and \cref{asmp:main} hold and for all $i \in [n]$, $\norm{\vx_i}_2 = 1$  and $\abs{\vy_i} \le C$ for some constant $C$. if we set the number of hidden nodes
\begin{align*}
    m = \Omega\left(\frac{n^5\norm{\vtildey}_2}{\lambda_0^4\delta^2} +  \frac{n_s^6}{\lambda_{0_s}^4 \kappa^2 \delta^3}\right),
\end{align*} and we set the step sizes $\eta_T = O\left(\frac{\lambda_0}{n^2}\right)$, $\eta_S =O\left(\frac{\lambdas}{n_S^2}\right)$ then with probability at least $1-2\delta$ over the random initialization we have for $t=0,1,2,\ldots$
\begin{align}
&\norm{\vy -\vu(t)}_2^2 \le \left(1-\frac{\eta \lambda _0}{2}\right)^{t}\norm{\tilde{\vy}}_2^2; \label{eq:du_2018_convergene}\\
&\norm{\vw_r(t) - \vw_r(0)} \le \frac{4\sqrt{n}\norm{\vtildey}}{\sqrt{m}\lambda_0}, \quad \forall r \in [m]. \nonumber
\end{align}
\end{theorem}
\begin{proof}[Proof of \cref{thm:du2018}]
We follow the exact proof as in \cite{du2018gradient}, with the exception of using \cref{lem:flipped_neurons_bound} instead of Lemma 4.1, and \cref{lem:close_to_init_small_perturbation} instead of Lemma 3.2.

The lower bound for $m$ is derived from the requirement on the constant $R$ that bounds the distance of $\vw_r(t)$ from the random initialization at $t=\init$. Notice that:
\begin{align*}
    \norm{\vw_r(t) - \vw_r(\init)} \le \norm{\vw_r(0) - \vw_r(\init)} + \norm{\vw_r(t) - \vw_r(0)}, \quad \forall r \in [m],
\end{align*}
where the bound for the left expression on the R.H.S is given by with probability $1 - \delta$ by \cref{cor:exact_pretrain_distance}.

The bound for the right expression on the R.H.S is given as a corollary of \eqref{eq:du_2018_convergene}:
\begin{align*}
    \norm{\vw_r(t) - \vw_r(0)} &\le \eta\sum_{s=0}^{t-1}\norm{\frac{\partial L\left(\vX, \vTheta(s)\right)}{\partial\vw_r(s)}}
    \le \eta\sum_{s=0}^{t}\frac{\sqrt{n}\norm{\vy -\vu(s)}}{\sqrt{m}}\\
    &\le  \eta\sum_{s=0}^{t}\frac{\sqrt{n}\left(1-\frac{n\lambda_0}{2}\right)^{s/2}}{\sqrt{m}}\norm{\vy -\vu(s)}\\
    &\le \eta\sum_{s=0}^{\infty}\frac{\sqrt{n}\left(1-\frac{n\lambda_0}{2}\right)^{s/2}}{\sqrt{m}}\norm{\vy -\vu(s)} = \frac{4\sqrt{n}\norm{\vtildey}}{\sqrt{m}\lambda_0}.
\end{align*}

Hence we require $R = \frac{C n_S}{\sqrt{m\delta} \lambdas } + \frac{4 \sqrt{n}\norm{\vtildey}}{\sqrt m \lambda_0 }$. From this requirement we derive the lower bound for $m$. 
\end{proof}

Using \cref{cor:exact_pretrain_distance} and \cref{thm:du2018} we obtain a the following corollary:
\begin{corollary}\label{cor:exact_finetune_distance}
   Assume \cref{asmp:pretrain_init}, \cref{asmp:main} and \cref{asmp:width_kappa_eta} hold, exists $C>0$ s.t. with probability at least $1 - 2\delta$ over the random initialization at time $t=\init$ we have
   \begin{align*}
        \norm{\vw_r(t) - \vw_r(\init)}_2 &\le \norm{\vw_r(0) - \vw_r(\init)}_2 +\norm{\vw_r(t) - \vw_r(0)}_2\\
        &\le  \frac{C n_S}{\sqrt{m\delta} \lambdas } + \frac{4 \sqrt{n}\norm{\vtildey}}{\sqrt m \lambda_0 }\quad \forall r\in[m].
    \end{align*}
\end{corollary}

Restate Lemma C.2 and Lemma C.3 from \cite{arora2019fine}:
\begin{lemma}[Adaption of Lemma C.2 from \cite{arora2019fine}]\label{lem:H_t_to_init}
Under the same setting as \cref{thm:du2018}, with probability at least $1-8\delta$ over the random initialization, for all $t \ge 0$ we have:
\begin{align*}
    &\norm{\vH(0) - \vH(\init)}_F = O\left(\frac{n^2 n_S}{\sqrt{m}\delta^{\nicefrac{3}{2}} \lambdas\kappa }\right),\\
    &\norm{\vH(t) - \vH(\init)}_F = O\left(\frac{n^2 n_S}{\sqrt{m}\delta^{\nicefrac{3}{2}} \lambdas\kappa } + \frac{n^{\nicefrac{5}{2}}\norm{\vtildey}}{\sqrt m \lambda_0\kappa\delta }\right),\\
	&\norm{\vZ(t) - \vZ(0)}_F = O\left(\sqrt{\frac{n n_S}{\sqrt{m}\delta^{\nicefrac{3}{2}}\kappa \lambdas} + \frac{n^{\nicefrac{3}{2}}\norm{\vtildey}}{\sqrt m \lambda_0\kappa\delta }}\right),
\end{align*}
for $\vZ(t) \triangleq \frac{1}{m}\sum_{i=1}^{n}{\sum_{r=1}^{m}{\indict\left\{\vw_r^\top(t)\vx_i >0\right\}}}$.
\end{lemma}
\begin{proof}
    For the first and seconds equality we use the exact proof of Lemma C.2 from \cite{arora2019fine}, replacing the value of $R$ with $\frac{C n_S}{\sqrt{m\delta} \lambdas }$ and $\frac{C n_S}{\sqrt{m\delta} \lambdas } + \frac{4 \sqrt{n}\norm{\vtildey}}{\sqrt m \lambda_0 }$ respectively (by using \cref{cor:exact_pretrain_distance} and \cref{cor:exact_finetune_distance} to bound the norm of the distance of each weight from initialization).
    The third equality also follows the same lines, with the difference being in:
    \begin{align*}
        \E\left[\norm{\vZ(t) - \vZ(0)}_F^2\right] & \le \frac{1}{m}\sum_{i=1}^{n}{\sum_{r=1}^{m}{\E\left[\indict\{A_{r,i}\}+\indict\{\norm{\vw_r(t)-\vw_r(0)}>\frac{4\sqrt{n} \norm{\vtildey}}{\sqrt{m} \lambda_0}\}\right]}}\\
        &\le \frac{1}{m}\cdot mn\cdot\frac{2R}{\sqrt{2\pi}\kappa} + \frac{n}{m}\delta.
    \end{align*}
The last pass is justified due to the bound on $\norm{\vw_r(t)-\vw_r(0)}$ for all $r \in [m]$ with probability $1-\delta$ from \cref{thm:du2018}. The wanted result is obtained, again, by plugging the R.H.S of \cref{cor:exact_finetune_distance} instead of $R$.
\end{proof}

\begin{lemma}[Lemma C.3 from \cite{arora2019fine}]\label{lem:H_infty_to_init}
with probability at least $1-\delta$, we have $\norm{\vH(\init) - \vH^\infty} = O\left(\frac{n\sqrt{\log{\frac{n}{\delta}}}}{\sqrt{m}}\right)$.
\end{lemma}


Using the results above, the wanted results of this section follows:
\begin{corollary}\label{cor:H_dist_bound}
   Under the same setting as \cref{thm:du2018}, with probability at least $1 - 9\delta$ over the random initialization we have have
   \begin{align*}
       &\norm{\vH(t) - \vH^\infty} =  O\left(\frac{n^2 n_S}{\sqrt{m}\delta^{\nicefrac{3}{2}} \lambdas\kappa } + \frac{n^{\nicefrac{5}{2}}\norm{\vtildey}}{\sqrt m \lambda_0\kappa\delta }\right),\\
       &\norm{\vH(0) - \vH^\infty} =  O\left(\frac{n^2 n_S}{\sqrt{m}\delta^{\nicefrac{3}{2}} \lambdas\kappa }\right).
   \end{align*}
\end{corollary}
\begin{proof}
    This corollary is direct by bounding $\norm{\vH(t) - \vH^\infty} \le \norm{\vH(\init) - \vH^\infty} + \norm{\vH(t) - \vH(\init)}$ and using \cref{lem:H_infty_to_init} and \cref{lem:H_t_to_init} to bound the R.H.S for the general $t>0$ case and for $t=0$.
\end{proof}

\subsection{Bound the distance from initialization}

Write the eigen-decomposition
\begin{align*}
    \mat H^\infty = \sum_{i=1}^n \lambda_i \vv_i \vv_i^\top,
\end{align*}
where $\vect v_1, \ldots, \vect v_n \in \R^n$ are orthonormal eigenvectors of $\mat H^\infty$ and $\lambda_1, \ldots, \lambda_n$ are corresponding eigenvalues. also define
\begin{align*}
    \indict_{i,r}(t) \triangleq \indict\left\{\vw_r^\top(t)\vx_i \ge 0\right\}.
\end{align*}
\begin{theorem}[Adaption of Theorem 4.1 from \cite{arora2019fine}]\label{thm:arora_convergence}
Assume \cref{asmp:main}, and suppose
 $m = \Omega\left( \frac{n^{5}\norm{\tilde{\vy}}_2^4}{\epsilon^2\kappa^2\delta^2\lambda_0^4} + \frac{n^4 n_{s}^2\norm{\tilde{\vy}}_2^2}{\epsilon^2\lambda_{0_s}^2\lambda_0^2\kappa^2\delta^{3}}  \right)$. 
Then with probability at least $1-\delta$ over the random initialization before pretraining ($t=\init$), for all $t=0, 1, 2, \ldots$ we have: 
\begin{align} \label{eqn:u(t)-y_size}
\norm{\vy-\vu(t)}_2
= \sqrt{\sum_{i=1}^{n}(1-\eta\lambda_i)^{2t} \left(\vect{v}_i^\top \tilde{\vy}\right)^2} \pm \epsilon.
\end{align}
\end{theorem}
We first note the important difference between this result and the original theorem is in the treatment of $\vu(0)$, the predictions of the model at $t=0$. While the original theorem shows that these predictions could be treated as negligible noise (for large enough $m$), we instead use them as part of the bound to the convergence of the training loss.
\begin{proof}
The core of our proof is to show that when $m$ is sufficiently large, the sequence $\left\{\vu(t) \right\}_{t=0}^\infty$ stays close to another sequence $\left\{\tilde{\vu}(t)\right\}_{t=0}^\infty$ which has a \emph{linear} update rule:
\begin{align}
    \tilde{\vu}(0) &= \vu(0), \nonumber \\
    \tilde{\vu}(t+1) &= \tilde{\vu}(t) - \eta \vH^\infty\left(\tilde{\vu}(t)-\vy\right). \label{eqn:u_hat_dynamics}
\end{align}
From~\eqref{eqn:u_hat_dynamics} we have
\begin{align*}
    \tilde{\vu}(t+1) - \vy = (\mat I - \eta \vH^\infty) \left(\tilde{\vu}(t)-\vy\right),
\end{align*}
which implies
\begin{align*}
    \tilde{\vu}(t)-\vy = (\mat I - \eta \vH^\infty)^t \left(\tilde{\vu}(0)-\vy\right) = - (\mat I - \eta \vH^\infty)^t \vtildey.
\end{align*}
Note that $(\mat I - \eta \vH^\infty)^t$ has eigen-decomposition
\begin{align*}
    (\mat I - \eta \vH^\infty)^t = \sum_{i=1}^n(1-\eta\lambda_i)^t \vv_i \vv_i^\top
\end{align*} and that $\vtildey$ can be decomposed as 
\begin{align*}
    \vtildey = \sum_{i=1}^n (\vv_i^\top \vtildey) \vv_i.
\end{align*}
Then we have
\begin{align*}
    \tilde{\vu}(t)-\vy = - \sum_{i=1}^n(1-\eta\lambda_i)^t (\vv_i^\top \vtildey) \vv_i,
\end{align*}
which implies
\begin{align}
    \norm{\tilde{\vu}(t)-\vy}_2^2 =  \sum_{i=1}^n (1-\eta\lambda_i)^{2t} (\vv_i^\top \vtildey)^2.\label{eq:tilde_sequence_result}
\end{align}

To prove that the two sequences stay close, we follow the exact proof of Theorem 4.1 in Appendix C of \cite{arora2019fine}. We start by observing the difference between the predictions at two successive steps:
\begin{align}
    \vu_i(t+1) - \vu_i(t) = \frac{1}{\sqrt{m}}\sum_{r=1}^{m}{\va_r\left[\sigma\left(\vw_r(t+1)^\top\vx_i\right) - \sigma\left(\vw_r(t)^\top\vx_i\right)\right]}. \label{eq:prediction_diff_t}
\end{align}
For each $i \in [n]$, divide the $m$ neurons into two parts: the neurons that can change their activation pattern of data-point $\vx_i$ during optimization and
those which can't. Since $|\vx_i|\le 1$, a neuron cannot change its activation pattern with respect to $\vx_i$ if $|\vx_i^\top\vw_r(\init)| > R$ and $|\vw_r(t) -\vw_r(\init)| \le R$ for the value of $R$ in \cref{cor:exact_finetune_distance}. Define the indices of the neurons in this group (i.e. cannot change their activation pattern...) as as $\bar{S}_i$, and the indices of the complementary group as $S_i$.

From \cref{lem:flipped_neurons_bound} we know that with probability $1-\delta$, for $R = \left(\frac{n_S}{\sqrt{m\delta} \lambdas } + \frac{ \sqrt{n}\norm{\vtildey}}{\sqrt m \lambda_0 }\right)$
\begin{align}
    |\bar{S}_i| \le O\left(\frac{m n}{\kappa\delta}\left(\frac{n_S}{\sqrt{m\delta} \lambdas } + \frac{ \sqrt{n}\norm{\vtildey}}{\sqrt m \lambda_0 }\right)\right) \label{eq:s_bar_bound}.
\end{align}

Following the same steps as in \cite{arora2019fine} and notice that \eqref{eq:prediction_diff_t} can be treated as:
\begin{align}
    \vu(t+1) - \vu(t) = -\eta\vH(t)\left(\vu(t) - \vy\right) + \vect{\epsilon}(t), \label{eq:prediction_diff_H_t}
\end{align}
where:
\begin{align*}
    \vect{\epsilon}_i(t) \triangleq &\frac{1}{\sqrt{m}}\sum_{r \in \bar{S}_i}{\left[\sigma\left(\vw_r(t+1)^\top\vx_i\right) - \sigma\left(\vw_r(t)^\top\vx_i\right)\right]} \\
    &+ \frac{\eta}{m}\sum_{j=1}^{n}{(u_j(t) - y_j)\vx_j^\top\vx_i\sum_{r \in \bar{S}_i}\indict_{r,i}(t)\indict_{r, j}(t)}.
\end{align*}
Next use \eqref{eq:s_bar_bound} to bound $\norm{\vect{\epsilon}(t)}$:
\begin{align*}
	\norm{\vect{\epsilon}(t)}_2 
	&\le \norm{\vect{\epsilon}(t)}_1 
	\le \sum_{i=1}^n \frac{2 \eta \sqrt n |\bar{S_i}|}{m} \norm{\vu(t) - \vy}_2\\
	&= O\left( \frac{\sqrt{m}n^{3/2}}{\kappa\delta^{3/2}}\left(\frac{\sqrt{\delta}\norm{\tilde{\vy}}_2}{\lambda_0} + \frac{n_s}{\sqrt{n}\lambda_{0_{s}}} \right) \right)   \frac{2 \eta \sqrt n }{m} \norm{\vu(t) - \vy}_2 \\
	&= O\left( \frac{\eta n^2}{\sqrt m \kappa\delta^{3/2}}\left(\frac{\sqrt{\delta}\norm{\tilde{\vy}}_2}{\lambda_0} + \frac{n_s}{\sqrt{n}\lambda_{0_{s}}} \right) \right)   \norm{\vu(t) - \vy}_2.
\end{align*}

Notice from \cref{cor:H_dist_bound} that $\vH(t)$ stays close to $\vH^\infty$. Then it is possible to rewrite \cref{eq:prediction_diff_H_t} as
\begin{align}
    \vu(t+1) - \vu(t) = -\eta\vH^\infty\left(\vu(k) - \vy\right) + \vect{\zeta}(t), \label{eq:prediction_diff_H_infty}
\end{align}
where $\vect{\zeta}(t) = -\eta\left(\vH^\infty - \vH(t)\right)\left(\vu(k) - \vy\right) + \vect{\epsilon}(t)$. Using \cref{cor:H_dist_bound} it follows that


\begin{align} 
    \norm{\vect{\zeta}(t)}_2 &\le \eta \norm{\vH^\infty - \vH(t)}_2 \norm{\vu(t) - \vy}_2 + \norm{\bm{\epsilon}(t)}_2 \nonumber\\
    &= O\left(\frac{\eta n^{5/2}\norm{\tilde{\vy}}_2}{\sqrt{m}\kappa\delta\lambda_0} + \frac{\eta n^2 n_{s}}{\sqrt{m}\lambda_{0_s}\kappa\delta^{3/2}} \right)\norm{\vu(t) - \vy}_2 \nonumber\\
    &+ O\left( \frac{\eta n^2}{\sqrt m \kappa\delta^{3/2}}\left(\frac{\sqrt{\delta}\norm{\tilde{\vy}}_2}{\lambda_0} + \frac{n_s}{\sqrt{n}\lambda_{0_{s}}} \right) \right)   \norm{\vu(t) - \vy}_2 \nonumber\\
    &= O\left(\frac{\eta n^{5/2}\norm{\tilde{\vy}}_2}{\sqrt{m}\kappa\delta\lambda_0} + \frac{\eta n^2 n_{s}}{\sqrt{m}\lambda_{0_s}\kappa\delta^{3/2}} \right)\norm{\vu(t) - \vy}_2 \label{eq:zeta}.
\end{align}
Apply \eqref{eq:prediction_diff_H_infty} recursively and get:
\begin{align}
\label{eq:recursive}
    \vu(t) - \vy = -\left(\vI -\eta\vH^\infty\right)^t\vtildey +\sum_{s=0}^{t-1}{\left(\vI -\eta\vH^\infty\right)^t\vect{\zeta}(t-1-s)}.
\end{align}
For the left term in \eqref{eq:recursive} we've shown in \eqref{eq:tilde_sequence_result} that: 
\begin{align*}
    \norm{-(\mat I - \eta \mat H^\infty)^t (\vtildey)}_2 = \sqrt{\sum_{i=1}^n (1-\eta\lambda_i)^{2t} (\vv_i^\top \vtildey)^2}.
\end{align*}
The right term in  \eqref{eq:recursive} can be bounded using \eqref{eq:zeta}:
\begin{align*}
	\norm{\sum_{s=0}^{t-1} (\mat I - \eta \mat H^\infty)^s \bm{\zeta}(t-1-s)}_2
	&\le \sum_{s=0}^{t-1} \norm{ \mat I - \eta \mat H^\infty}_2^s \norm{\bm{\zeta}(t-1-s)}_2 \\
	&\le \sum_{s=0}^{t-1} (1-\eta \lambda_0)^s O\left(\frac{\eta n^{5/2}\norm{\tilde{\vy}}_2}{\sqrt{m}\kappa\delta\lambda_0} + \frac{\eta n^2 n_{s}}{\sqrt{m}\lambda_{0_s}\kappa\delta^{3/2}} \right) \norm{\vect u(t-1-s) - \vy}_2  \\
	&\le \sum_{s=0}^{t-1} (1-\eta \lambda_0)^s O\left(\frac{\eta n^{5/2}\norm{\tilde{\vy}}_2}{\sqrt{m}\kappa\delta\lambda_0} + \frac{\eta n^2 n_{s}}{\sqrt{m}\lambda_{0_s}\kappa\delta^{3/2}} \right) \left( 1 - \frac{\eta\lambda_0}{4} \right)^{t-1-s} \norm{\tilde{\vy}}_2  \\
	&\le t \left( 1 - \frac{\eta\lambda_0}{4} \right)^{t-1}   O\left(\frac{\eta n^{5/2}\norm{\tilde{\vy}}_2^2}{\sqrt{m}\kappa\delta\lambda_0} + \frac{\eta n^2 n_{s}\norm{\tilde{\vy}}_2}{\sqrt{m}\lambda_{0_s}\kappa\delta^{3/2}} \right) .
\end{align*}
Combining all of the above it follows:
\begin{align*}
	\norm{\vu(t) - \vy}_2
	&= \sqrt{\sum_{i=1}^n (1-\eta\lambda_i)^{2t} (\vv_i^\top \tilde{\vy})^2} \pm O\left( t \left( 1 - \frac{\eta\lambda_0}{4} \right)^{t-1} \left(\frac{\eta n^{5/2}\norm{\tilde{\vy}}_2^2}{\sqrt{m}\kappa\delta\lambda_0} + \frac{\eta n^2 n_{s}\norm{\tilde{\vy}}_2}{\sqrt{m}\lambda_{0_s}\kappa\delta^{3/2}} \right)\right) \\
	&= \sqrt{\sum_{i=1}^n (1-\eta\lambda_i)^{2t} (\vv_i^\top \tilde{\vy})^2} \pm O\left(\frac{n^{5/2}\norm{\tilde{\vy}}_2^2}{\sqrt{m}\kappa\delta\lambda_0^2} + \frac{n^2 n_{s}\norm{\tilde{\vy}}_2}{\sqrt{m}\lambda_{0_s}\lambda_0\kappa\delta^{3/2}} \right).
	\end{align*}
	where we used $\max\limits_{t\ge 0} \left\{ t (1-\eta\lambda_0/4)^{t-1} \right\} = O(1/(\eta\lambda_0))$.
	From the choices of $\kappa$ and $m$, the above error term is at most $\epsilon$.
	This completes the proof of \cref{thm:arora_convergence}.
\end{proof}

\subsection{Deriving a population risk bound}

Before proving Theorem 6.1 from the main text, we start by stating and proving some Lemmas:

\begin{lemma}\label{lem:distance_bounds}
	Suppose $m \ge \kappa^{-2} \poly\left(\norm{\tilde{\vy}}_2, n, n_s, \lambda_0^{-1}, \lambda_{0_s}^{-1}, \delta^{-1} \right)  $ and $\eta = O\left( \frac{\lambda_0}{n^2} \right)$.
	Then with probability at least $1-\delta$ over the random initialization at $t=\init$, we have for all $t\ge0$:
	\begin{itemize}
		\item $\norm{\vw_r(t) - \vw_r(0)}_2 =O\left( \frac{ \sqrt{n} \norm{\tilde{\vy}}_2 }{\sqrt m \lambda_0} \right)$ $(\forall r\in[m])$, and
		\item $\norm{\vW(t)-\vW(0)}_{F} \le \sqrt{\tilde{\vy}^\top \left(\vH^{\infty}\right)^{-1}\tilde{\vy}} + \frac{\poly\left( \norm{\tilde{\vy}}_2, n, n_s, \frac{1}{\lambda_0}, \frac{1}{\lambda_{0_s}}, \frac{1}{\delta} \right)}{m^{1/4} \kappa^{1/2}} $.
	\end{itemize}
\end{lemma}
\begin{proof}
The bound on the movement of each $\vw_r$ is proven in \cref{thm:du2018}. The second bound is achieved by coupling the trajectory of $\left\{\vW(t) \right\}_{k=0}^\infty$ with another simpler trajectory $\left\{\widetilde{\mat{W}}(t)\right\}_{k=0}^\infty$ defined as:
\begin{align}
    \widetilde{\mat{W}}(0) =\,& \vW(0), \nonumber\\
    \vectorize{\widetilde{\vW}(t+1)} =\,& \vectorize{\widetilde{\mat{W}}(t)} \label{eqn:W_tilde_traj} \\&- \eta\mat{Z}(0)\left( \mat{Z}(0)^\top \vectorize{\widetilde{\mat{W}}(t)}-\vy\right). \nonumber
\end{align}

First we give a proof of  $\norm{\widetilde{\mat W}(\infty) - \widetilde{\mat W}(0)}_F = \sqrt{\vtildey^\top \mat H(0)^{-1} \vtildey}$ as an illustration for the proof of Lemma~\ref{lem:distance_bounds}.
Define $\vv(t) = \mat Z(0)^\top \vectorize{\widetilde{\mat W}(t)} \in \R^n$.
Then from~\eqref{eqn:W_tilde_traj} we have $\vect v(0)=\mat Z(0)^\top \vectorize{\vW(0)}$ and $\vect v(k+1) = \vv(t) - \eta \mat H(0) (\vv(t) - \vy)$, yielding
$\vv(t) - \vy = - (\mat I - \eta \mat H(0))^t \tilde{\vy}$.
Plugging this back to~\eqref{eqn:W_tilde_traj} we get
$\vectorize{\widetilde{\vW}(t+1)} - \vectorize{\widetilde{\mat{W}}(t)} = \eta \mat Z(0) (\mat I - \eta \mat H(0))^t \tilde{\vy}$.
Then taking a sum over $k=0, 1, \ldots$ we have
\vspace{-0.2cm}
\begin{align*}
\vectorize{\widetilde{\mat{W}}(\infty)} - \vectorize{\widetilde{\mat{W}}(0)} &= \sum_{k=0}^\infty \eta \mat Z(0) (\mat I - \eta \mat H(0))^k \tilde{\vy} \\
&= \mat Z(0) \mat H(0)^{-1} \tilde{\vy}.
\end{align*}
\vspace{-0.2cm}
The desired result thus follows:
\begin{align*}
\norm{\widetilde{\mat W}(\infty) - \widetilde{\mat W}(0)}_F^2
&= \tilde{\vy}^\top \mat H(0)^{-1} \mat Z(0)^\top \mat Z(0) \mat H(0)^{-1} \tilde{\vy} \\
&= \tilde{\vy}^\top \mat H(0)^{-1} \tilde{\vy}.
\end{align*}

Now we bound the difference between the trajectories.
Recall the update rule for $\mat W$:
	\begin{align} \label{eqn:gd-repeated}
	\vectorize{\vW(t+1)} = \vectorize{\vW(t)} - \eta \mat Z(t) (\vu(t) - \vy).
	\end{align}
Follow the same steps from Lemma 5.3 from \cite{arora2019fine}, using the results from \cref{thm:arora_convergence} when needed to obtain the proof for this lemma.
According to the proof of \cref{thm:arora_convergence}  we can write
\begin{align} \label{eqn:u-expression-2}
	\vu(t) - \vy = -(\mat I - \eta \mat H^\infty)^t \tilde{\vy} + \vect e(t),
	\end{align}
	where
\begin{align} \label{eqn:e(t)-bound}
	\norm{\vect e(t)} = O\left(t\left(1-\frac{\eta \lambda_0}{4}\right)^{t-1}\cdot \left(\frac{\eta n^{5/2}\norm{\tilde{\vy}}_2^2}{\sqrt{m}\kappa\delta\lambda_0} + \frac{\eta n^2 n_{s}\norm{\tilde{\vy}}_2}{\sqrt{m}\lambda_{0_s}\kappa\delta^{3/2}} \right)  \right).
\end{align}
Plugging \eqref{eqn:u-expression-2} into \eqref{eqn:gd-repeated} and taking a sum over $t = 0, 1, \ldots, T-1$, we get:
\begin{align}
	&\vectorize{\vW(T)} - \vectorize{\vW(0)} \nonumber \\
	=\,& \sum_{t=0}^{T-1} \left( \vectorize{\vW(t+1)} - \vectorize{\vW(t)}  \right) \nonumber \\
	=\,& - \sum_{t=0}^{T-1} \eta \mat Z(t) (\vu(t) - \vy) \nonumber \\
	=\,&  \sum_{t=0}^{T-1} \eta \mat Z(t) \left((\mat I - \eta \mat H^\infty)^t \tilde{\vy} - \vect e(t)\right) \nonumber\\
	=\,&  \sum_{t=0}^{T-1} \eta \mat Z(t) (\mat I - \eta \mat H^\infty)^t \tilde{\vy}  - \sum_{t=0}^{T-1} \eta \mat Z(t) \vect e(t) \nonumber \\
	=\,&  \sum_{t=0}^{T-1} \eta \mat Z(0) (\mat I - \eta \mat H^\infty)^t \tilde{\vy} + \sum_{t=0}^{T-1} \eta (\mat Z(t) - \mat Z(0)) (\mat I - \eta \mat H^\infty)^t \tilde{\vy}  - \sum_{t=0}^{T-1} \eta \mat Z(t)  \vect e(t). \label{eqn:W-movement}
\end{align}
The second and the third terms in~\eqref{eqn:W-movement} are considered perturbations, and we can upper bound their norms easily.
	For the second term, from \cref{lem:close_to_init_small_perturbation} we get:
\begin{align} 
	& \norm{\sum_{t=0}^{T-1} \eta (\mat Z(t) - \mat Z(0)) (\mat I - \eta \mat H^\infty)^t \vy}_2 \nonumber \\
	\le\,& \sum_{t=0}^{T-1} \eta \cdot O\left( \sqrt{\frac{n^{3/2}\norm{\tilde{\vy}}_2}{\sqrt{ m}\kappa\delta\lambda_0}+\frac{n n_s}{\sqrt{ m}\kappa\lambda_{0_{s}}\delta^{3/2}}}\right) \norm{\mat I - \eta \mat H^\infty}_2^t \norm{\tilde{\vy}}_2 \nonumber \\
	\le\,&  O\left( \eta\sqrt{\frac{n^{3/2}\norm{\tilde{\vy}}_2}{\sqrt{ m}\kappa\delta\lambda_0}+\frac{n n_s}{\sqrt{ m}\kappa\lambda_{0_{s}}\delta^{3/2}}} \right)  \sum_{t=0}^{T-1} (1-\eta\lambda_0)^t \norm{\tilde{\vy}}_2 \nonumber \\
	=\,& O\left( \sqrt{\frac{n^{3/2}\norm{\tilde{\vy}}_2^3}{\sqrt{ m}\kappa\delta\lambda_0^3}+\frac{n n_s \norm{\tilde{\vy}}_2^2}{\sqrt{ m}\kappa\lambda_{0_{s}}\lambda_0^2\delta^{3/2}}} \right) \label{eqn:W-movement-perturbation-1}. 
\end{align}
For the third term we get:
\begin{align}
	&\norm{\sum_{t=0}^{T-1} \eta \mat Z(t)  \vect e(t)}_2 \nonumber\\
	\le\,& \sum_{t=0}^{T-1} \eta \sqrt n \cdot O\left(  t\left(1-\frac{\eta \lambda_0}{4}\right)^{t-1}\cdot \left(\frac{\eta n^{5/2}\norm{\tilde{\vy}}_2^2}{\sqrt{m}\kappa\delta\lambda_0} + \frac{\eta n^2 n_{s}\norm{\tilde{\vy}}_2}{\sqrt{m}\lambda_{0_s}\kappa\delta^{3/2}} \right)  \right) \nonumber \\
	=\,& O\left( \left(\frac{\eta^2 n^3\norm{\tilde{\vy}}_2^2}{\sqrt{m}\kappa\delta\lambda_0} + \frac{\eta^2 n^{5/2} n_{s}\norm{\tilde{\vy}}_2}{\sqrt{m}\lambda_{0_s}\kappa\delta^{3/2}} \right) \sum_{t=0}^{T-1} t\left(1-\frac{\eta \lambda_0}{4}\right)^{t-1} \right) \nonumber \\
	=\,& O\left( \left(\frac{\eta^2 n^3\norm{\tilde{\vy}}_2^2}{\sqrt{m}\kappa\delta\lambda_0} + \frac{\eta^2 n^{5/2} n_{s}\norm{\tilde{\vy}}_2}{\sqrt{m}\lambda_{0_s}\kappa\delta^{3/2}} \right) \cdot \frac{1}{\eta \lambda_0} \right) \nonumber \\
	=\,& O\left( \frac{\eta n^3\norm{\tilde{\vy}}_2^2}{\sqrt{m}\kappa\delta\lambda_0^2} + \frac{\eta n^{5/2} n_{s}\norm{\tilde{\vy}}_2}{\sqrt{m}\lambda_{0_s}\lambda_0\kappa\delta^{3/2}} \right) \label{eqn:W-movement-perturbation-2}.
\end{align}

Define $\mat K = \eta \sum_{t=0}^{T-1} (\mat I - \eta \mat H^\infty)^t$.
using $\norm{\mat H(0) - \mat H^\infty}_F = O\left(\frac{n^2 n_{s}}{\sqrt{m}\lambda_{0_s}\kappa\delta^{3/2}}\right)$ (\cref{cor:H_dist_bound}) we have

\begin{align}\label{eqn:W-movement-main-term-bound-interim}
&\norm{\sum_{t=0}^{T-1} \eta \mat Z(0) (\mat I - \eta \mat H^\infty)^t \tilde{\vy}}_2^2 \\
=\,& \norm{ \mat Z(0) \mat K  \tilde{\vy}}_2^2 \\
=\,&  \tilde{\vy}^\top \mat K \mat Z(0)^\top \mat Z(0) \mat K \tilde{\vy} \\
=\,&  \tilde{\vy}^\top \mat K \mat H(0)  \mat K \tilde{\vy} \\
\le\,&  \tilde{\vy}^\top \mat K \mat H^\infty  \mat K \tilde{\vy}  +    \norm{\mat H(0) - \mat H^\infty}_2  \norm{\mat K}_2^2 \norm{\tilde{\vy}}_2^2 \\
\le \,&  \tilde{\vy}^\top \mat K \mat H^\infty  \mat K \tilde{\vy}  +    O\left( \frac{n^2 n_{s}}{\sqrt{m}\lambda_{0_s}\kappa\delta^{3/2}} \right) \cdot  \left( \eta \sum_{t=0}^{T-1} (\mat I - \eta \lambda_0)^t\right)^2 \norm{\tilde{\vy}}_2^2 \\
= \,&  \tilde{\vy}^\top \mat K \mat H^\infty  \mat K \tilde{\vy}  +    O\left( \frac{n^2 n_{s} \norm{\tilde{\vy}}_2^2}{\sqrt{m}\lambda_{0_s}\lambda_0^2\kappa\delta^{3/2}} \right) .
\end{align}

Let the eigen-decomposition of $\mat H^\infty$ be $\mat H^\infty = \sum_{i=1}^n \lambda_i \vv_i \vv_i^\top$.
Since $\mat K$ is a polynomial of $\mat H^\infty$, it has the same set of eigenvectors as $\mat H^\infty$, and we have
\begin{align*}
    \mat K  = \sum_{i=1}^n \eta \sum_{t=0}^{T-1} (1 - \eta \lambda_i)^t \vv_i \vv_i^\top
    = \sum_{i=1}^n \frac{1 - (1 - \eta \lambda_i)^{T}}{\lambda_i} \vv_i \vv_i^\top.
\end{align*}
It follows that
\begin{align*}
    \mat K \mat H^\infty  \mat K
    = \sum_{i=1}^n \left( \frac{1 - (1 - \eta \lambda_i)^{T}}{\lambda_i} \right)^2 \lambda_i  \vv_i \vv_i^\top
    \preceq \sum_{i=1}^n \frac{1}{\lambda_i}\vv_i \vv_i^\top
    = \left( \mat H^\infty \right)^{-1}.
\end{align*}
Plugging this into \eqref{eqn:W-movement-main-term-bound-interim}, we get
\begin{align} \label{eqn:W-movement-main-term-bound}
	\norm{\sum_{t=0}^{T-1} \eta \mat Z(0) (\mat I - \eta \mat H^\infty)^t \tilde{\vy}_2}
	&\le \sqrt{\tilde{\vy}^\top (\mat H^\infty)^{-1} \tilde{\vy}  +    O\left( \frac{n^2 n_{s} \norm{\tilde{\vy}}_2^2}{\sqrt{m}\lambda_{0_s}\lambda_0^2\kappa\delta^{3/2}} \right) }\\
	&\le \sqrt{\tilde{\vy}^\top (\mat H^\infty)^{-1} \tilde{\vy}}  +    O\left( \sqrt{\frac{n^2 n_{s} \norm{\tilde{\vy}}_2^2}{\sqrt{m}\lambda_{0_s}\lambda_0^2\kappa\delta^{3/2}} } \right) .
\end{align}
	
Finally, plugging the three bounds \eqref{eqn:W-movement-perturbation-1}, \eqref{eqn:W-movement-perturbation-2} and \eqref{eqn:W-movement-main-term-bound} into \eqref{eqn:W-movement}, we have
\begin{align*}
& \norm{\vW(T) - \vW(0)}_F \\
=\,& \norm{\vectorize{\vW(T)} - \vectorize{\vW(0)} }_2 \\
\le\,& \sqrt{\tilde{\vy}^\top (\mat H^\infty)^{-1} \tilde{\vy}}  +    O\left( \sqrt{\frac{n^2 n_{s} \norm{\tilde{\vy}}_2^2}{\sqrt{m}\lambda_{0_s}\lambda_0^2\kappa\delta^{3/2}}} \right)
	+ O\left( \sqrt{\frac{n^{3/2}\norm{\tilde{\vy}}_2^3}{\sqrt{ m}\kappa\delta\lambda_0^3}+\frac{n n_s \norm{\tilde{\vy}}_2^2}{\sqrt{ m}\kappa\lambda_{0_{s}}\lambda_0^2\delta^{3/2}}} \right)\\
	&+ O\left( \frac{\eta n^3\norm{\tilde{\vy}}_2^2}{\sqrt{m}\kappa\delta\lambda_0^2} + \frac{\eta n^{5/2} n_{s}\norm{\tilde{\vy}}_2}{\sqrt{m}\lambda_{0_s}\lambda_0\kappa\delta^{3/2}}  \right) \\
=\,& \sqrt{\tilde{\vy}^\top (\mat H^\infty)^{-1} \tilde{\vy}} + \frac{\poly\left( \norm{\tilde{\vy}}_2, n, n_s, \frac{1}{\lambda_0}, \frac{1}{\lambda_{0_s}}, \frac{1}{\delta} \right)}{m^{1/4} \kappa^{1/2}} .
\end{align*}
This finishes the proof of Lemma~\ref{lem:distance_bounds}.
\end{proof}

\begin{lemma}\label{lem:rad_dist_func_class}
Given $R>0$,
with probability at least $1-\delta$ over the random initialization ($\mat W(\init),\vect a)$, simultaneously for every $B>0$, the following function class
\begin{align*}
    \cF^{\vW(0),\va}_{R,B}  = \{ f_{\mat W} : \norm{\vect w_r - \vw_r(0)}_2 \le R \, (\forall r\in[m]), \\ \norm{\mat W - \vW(0)}_F \le B \}
\end{align*}
has empirical Rademacher complexity bounded as:
\begin{align*}
	&\calR_S\left( \cF^{\vW(0),\va}_{R,B} \right)
   =  \frac{1}{n} \expect_{\bm \eps \in \{\pm1\}^n}\left[\sup_{f \in \cF^{\vW(0),\va}_{R,B}}\sum_{i=1}^{n} \eps_i f(\vx_i)\right]
	\\
	\le \,&	\frac{B}{\sqrt{n}} + \frac{2 R(R + \frac{C n_s}{\sqrt{m\delta}\lambdas}) \sqrt m}{ \kappa} + R \sqrt{2\log \frac2\delta}.
\end{align*}
\end{lemma}
\begin{proof}
We need to upper bound
\begin{align*}
	\calR_S\left(\cF^{\vW(0),\va}_{R,B} \right) 
	&= \frac{1}{n} \Erc{\sup_{f \in \cF^{\vW(0),\va}_{R,B}}\sum_{i=1}^{n} \eps_i f(\vx_i)} \\
	&= \frac{1}{n} \Erc{ \sup_{\genfrac{}{}{0pt}{1}{\mat W: \norm{\mat W - \vW(0)}_{2, \infty} \le R}{\norm{\mat W - \vW(0)}_F \le B }} \sum_{i=1}^{n}{ \eps_i \sum_{r=1}^m \frac{1}{\sqrt m} a_r \sigma(\vect w_r^\top \vect x_i)}},
\end{align*}
where $\norm{\mat W - \vW(0)}_{2, \infty} = \max\limits_{r\in[m]}\norm{\vw_r-\vw_r(0)}_2$.

Similar to the proof of \cref{lem:flipped_neurons_bound}, we define events:
\begin{align*}
    \tilde{A}_{r, i} \triangleq \left\{ \abs{\vw_{r}(0)^\top \vx_i} \le R \right\}, \quad i\in[n], r\in[m].
\end{align*}
Since we only look at $\mat W$ such that $\norm{\vw_r-\vw_r(0)}_2\le R$ for all $ r\in[m]$, if $\indict\{\tilde{A}_{r, i}\}=0$ we must have $\indict\{ \vw_r^\top \vx_i >0 \} = \indict\{\vw_r(0) \vx_i\ge  0\} = \indict_{r, i}(0)$. Thus we have:
\begin{align*}
    \bone{\neg \tilde{A}_{r, i}} \relu{\vw_r^\top \vx_i}
    =\bone{\neg \tilde{A}_{r, i}} \indict_{r, i}(0)\vw_r^\top \vx_i,
\end{align*}

It follows that:
    \begin{align*}
    &\sum_{i=1}^n \eps_i \sum_{r=1}^m a_r \relu{\vw_r^\top \vx_i} -\sum_{i=1}^n \eps_i \sum_{r=1}^m a_r \indict_{r,i}(0) \vw_r^\top \vx_i   \\
    =\,& \sum_{r=1}^m \sum_{i=1}^n \left( \bone{\tilde{A}_{r, i}} + \bone{\neg \tilde{A}_{r, i}}\right) \eps_i a_r \left( \relu{\vw_r^\top \vx_i} - \indict_{r,i}(0) \vw_r^\top \vx_i   \right)\\
    =\,& \sum_{r=1}^m \sum_{i=1}^n  \bone{\tilde{A}_{r, i}}  \eps_i a_r \left( \relu{\vw_r^\top \vx_i} - \indict_{r,i}(0) \vw_r^\top \vx_i   \right)\\
    =\,& \sum_{r=1}^m \sum_{i=1}^n  \bone{\tilde{A}_{r, i}}  \eps_i a_r \left( \relu{\vw_r^\top \vx_i} - \indict_{r,i}(0) \vw_r(0)^\top \vx_i    - \indict_{r,i}(0) (\vw_r-\vw_r(0))^\top \vx_i  \right)\\
    =\,& \sum_{r=1}^m \sum_{i=1}^n  \bone{\tilde{A}_{r, i}}  \eps_i a_r \left( \relu{\vw_r^\top \vx_i} - \relu{\vw_r(0)^\top \vx_i}   - \indict_{r,i}(0) (\vw_r-\vw_r(0))^\top \vx_i   \right)\\
    \le\, & \sum_{r=1}^m \sum_{i=1}^n  \bone{\tilde{A}_{r, i}} \cdot 2R.
\end{align*}

Thus we can bound the Rademacher complexity as:
    \begin{align*}
    \calR_S\left( \cF^{\vW(0),\va}_{R,B} \right)
    =\,& \frac1n \Erc{ \sup_{\tiny \substack{\mat W: \norm{\vW-\vW(0)}_{2,\infty}\le R \\ \norm{\vW-\vW(0)}_F\le B }}  \sum_{i=1}^n\eps_i \sum_{r=1}^m\frac{ a_r}{\sqrt{m}} \relu{\vw_r^\top \vx} }\\
    \le\, & \frac1n  \Erc{ \sup_{\tiny \substack{\mat W: \norm{\vW-\vW(0)}_{2,\infty}\le R \\ \norm{\vW-\vW(0)}_F\le B }}  \sum_{i=1}^n\eps_i  \sum_{r=1}^m\frac{ a_r}{\sqrt{m}} \indict_{r,i}(0) \vw_r^\top \vx_i   } + \frac{2R}{n\sqrt{m}}\sum_{r=1}^m \sum_{i=1}^n  \bone{\tilde{A}_{r, i}}  \\
    \le\, & \frac1n  \Erc{ \sup_{\mat W: \norm{\vW-\vW(0)}_F\le B}  \sum_{i=1}^n\eps_i   \sum_{r=1}^m\frac{ a_r}{\sqrt{m}} \indict_{r,i}(0) \vw_r^\top \vx_i  } + \frac{2R}{n\sqrt{m}}\sum_{r=1}^m \sum_{i=1}^n  \bone{\tilde{A}_{r, i}}  \\
    =\, & \frac1n  \Erc{ \sup_{\mat W: \norm{\vW-\vW(0)}_F\le B}  \vectorize{\mat W}^\top \mat Z(0) {\bm{\eps}} }+ \frac{2R}{n\sqrt{m}}\sum_{r=1}^m \sum_{i=1}^n  \bone{\tilde{A}_{r, i}}  \\
    =\, & \frac1n  \Erc{ \sup_{\mat W: \norm{\vW-\vW(0)}_F\le B}  \vectorize{\mat W - \vW(0)}^\top \mat Z(0) {\bm{\eps}} } + \frac{2R}{n\sqrt{m}}\sum_{r=1}^m \sum_{i=1}^n  \bone{\tilde{A}_{r, i}}  \\
    \le\,& \frac1n \Erc{B\cdot \norm{\mat Z(0) \bm{\eps}}_2} + \frac{2R}{n\sqrt{m}}\sum_{r=1}^m \sum_{i=1}^n  \bone{\tilde{A}_{r, i}}  \\
    \le\,& \frac Bn \sqrt{\Erc{ \norm{\mat Z(0) \bm{\eps}}_2^2}} + \frac{2R}{n\sqrt{m}}\sum_{r=1}^m \sum_{i=1}^n  \bone{\tilde{A}_{r, i}}  \\
    =\, & \frac Bn \norm{\vZ(0)}_F + \frac{2R}{n\sqrt{m}} \sum_{r=1}^m \sum_{i=1}^n  \bone{\tilde{A}_{r, i}} .
\end{align*}

Next we bound $\norm{\mat Z(0)}_F$ and $\sum_{r=1}^m \sum_{i=1}^n  \bone{\tilde{A}_{r, i}}$.

For $\norm{\mat Z(0)}_F$, notice that
\begin{align*}
\norm{\mat Z(0)}_F^2 = \frac1m \sum_{r=1}^m \left( \sum_{i=1}^n \indict_{r,i}(0) \right) \le n.
\end{align*}

Now observe the following lemma:
\begin{lemma}\label{lm:tilde_A}
    With probability $1 - \delta$, if $\abs{\vw_{r}(\init)^\top \vx_i} > R + \frac{C n_s}{\sqrt{m\delta}\lambdas}$ then $\indict\{\tilde{A}_{r, i}\} = 0$.
\end{lemma}
\begin{proof}
    From \cref{cor:exact_pretrain_distance} exists $C >0$ s.t. with probability $1-\delta$, for all $r \in [m]: \norm{\vw_r(0)-\vw_r(\init)} \le \frac{C n_s}{\sqrt{m\delta}\lambdas}$. From the triangle inequality:
    \begin{align*}
        \abs{\vw_{r}(0)^\top \vx_i} &\ge \norm{\vw_{r}(0)^\top \vx_i}\\
        &= \norm{\vw_{r}(\init)^\top \vx_i - \left(\vw_{r}(\init)-\vw_{r}(0)\right)^\top \vx_i}\\
        &\ge \norm{\vw_{r}(\init)^\top \vx_i} - \norm{\left(\vw_{r}(\init)-\vw_{r}(0)\right)^\top \vx_i}.
    \end{align*}
    Since $\norm{\vx}=1$, and with the same probability above:
    \begin{align*}
        \norm{\left(\vw_{r}(\init)-\vw_{r}(0)\right)^\top \vx_i} \le \frac{C n_s}{\sqrt{m\delta}\lambdas},
    \end{align*}
    thus
    \begin{align*}
        \abs{\vw_{r}(0)^\top \vx_i} &\ge \norm{\vw_{r}(\init)^\top \vx_i} - \norm{\left(\vw_{r}(\init)-\vw_{r}(0)\right)^\top \vx_i}\\
        &\ge \norm{\vw_{r}(\init)^\top \vx_i} - \frac{C n_s}{\sqrt{m\delta}\lambdas}\\
        &> R + \frac{C n_s}{\sqrt{m\delta}\lambdas} - \frac{C n_s}{\sqrt{m\delta}\lambdas} = R.
    \end{align*}
\end{proof}

For $\sum_{r=1}^m \sum_{i=1}^n  \bone{\tilde{A}_{r, i}}$, from \cref{lm:tilde_A} we notice that 
\begin{align*}
    \sum_{r=1}^m \sum_{i=1}^n  \bone{\tilde{A}_{r, i}} \le \sum_{r=1}^m \sum_{i=1}^n  \bone{A_{r, i}},
\end{align*} for $A_{r, i}$ being defined as in \cref{lem:flipped_neurons_bound}.
Since all $m$ neurons are independent at $t=\init$ and from \cref{lem:flipped_neurons_bound} and \cref{cor:exact_pretrain_distance} we know $\E\left[ \sum_{i=1}^n  \bone{A_{r, i}} \right] \le \frac{\sqrt 2 n (R + \frac{C n_s}{\sqrt{m\delta}\lambdas})}{\sqrt\pi \kappa}$. Then by Hoeffding's inequality, with probability at least $1-\delta/2$ we have
\begin{align*}
\sum_{r=1}^m \sum_{i=1}^n  \bone{\tilde{A}_{r, i}} \le \sum_{r=1}^m \sum_{i=1}^n  \bone{A_{r, i}} \le mn \left( \frac{\sqrt 2  (R + \frac{C n_s}{\sqrt{m\delta}\lambdas})}{\sqrt\pi \kappa} + \sqrt{\frac{\log \frac2\delta}{2m}} \right).
\end{align*}

Therefore, with probability at least $1-\delta$, the Rademacher complexity is bounded as:
\begin{align*}
\calR_S\left( \cF^{\vW(0),\va}_{R,B} \right)
&\le \frac Bn \left( \sqrt{n} \right) + \frac{2R}{n\sqrt{m}} mn \left( \frac{\sqrt 2  (R + \frac{C n_s}{\sqrt{m\delta}\lambdas})}{\sqrt\pi \kappa} + \sqrt{\frac{\log \frac2\delta}{2m}} \right) \\
&= \frac{B}{\sqrt{n}} + \frac{2\sqrt 2 R(R + \frac{C n_s}{\sqrt{m\delta}\lambdas}) \sqrt m}{\sqrt\pi \kappa} + R \sqrt{2\log \frac2\delta},
\end{align*}
completing the proof of Lemma~\ref{lem:rad_dist_func_class}.
(Note that the high probability events used in the proof do not depend on the value of $B$, so the above bound holds simultaneously for every $B$.)
\end{proof}
\subsection{Proof of Theorem 6.1 (Main Text)} \label{app:proof-thm:shallow_relu_generalization}

\begin{proof}[Proof of Theorem 6.1 (Main Text)]

	First of all, from \cref{asm:non-singular_supp} we have $\lambda_{\min}(\mat H^\infty) \ge \lambda_0$. The rest of the proof is conditioned on this happening. We follow exactly the same steps as in \cite{arora2019fine} with minor changes.
	
	From Theorem~\ref{thm:du2018}, Lemma~\ref{lem:distance_bounds} and Lemma~\ref{lem:rad_dist_func_class}, we know that for any sample $S$, with probability at least $1-\delta/3$ over the random initialization, the followings hold simultaneously:
	 \begin{enumerate}[(i)]
	 	\item Optimization succeeds (Theorem~\ref{thm:du2018}):
	 	\begin{align*} 
    	 	\frac{1}{2}\norm{\vtildey - \vu(t)} \le \left( 1-\frac{\eta\lambda_0}{2} \right)^t \cdot \norm{\tilde{\vy}}_2 \le \frac12. 
	 	\end{align*}
	 	This implies an upper bound on the training error $L(\vX; \vTheta(t)) = \frac1n \sum_{i=1}^n \ell(f_{\vW(t)}(\vect x_i), y_i) = \frac1n \sum_{i=1}^n \ell(u_i(t), y_i)$:
	 	\begin{align*}
    	 	L(\vX; \vTheta(t)) &= \frac1n \sum_{i=1}^n \left[ \ell(u_i(t), y_i) - \ell(y_i, y_i)\right] \le \frac1n \sum_{i=1}^n \abs{u_i(t) - y_i} \\
    	 	&\le \frac{1}{\sqrt n} \norm{\vu(t) - \vy}_2 = \sqrt{\frac{2\frac{1}{2}\norm{\vtildey - \vu(t)}}{n}} \le \frac{1}{\sqrt n}. 
	 	\end{align*}
	 	
	 	\item $\norm{\vw_r(t) - \vw_r(0)}_2 \le R$ $(\forall r\in[m])$ and
	 	 $\norm{\vW(t)-\vW(0)}_{F} \le B $,
	 	where $R= O\left( \frac{ \sqrt{n}\norm{\tilde{\vy}}_2}{\sqrt m \lambda_0} \right)$   and $B = \sqrt{\tilde{\vy}^\top \left(\vH^{\infty}\right)^{-1}\tilde{\vy}} + \frac{\poly\left( \norm{\tilde{\vy}}_2, n, n_s, \frac{1}{\lambda_0}, \frac{1}{\lambda_{0_s}}, \frac{1}{\delta} \right)}{m^{1/4} \kappa^{1/2}}$.
	 	Note that $B\le O\left( \sqrt{\frac{n}{\lambda_0}} \right)$.
	 	
	 	\item
	 	Let $B_i =  i$ ($i=1, 2, \ldots$).
	 	Simultaneously for all $i$, the function class $\cF_{R, B_i}^{\vW(0), \vect a}$ has Rademacher complexity bounded as
	 	\begin{align*}
	 	\calR_S\left( \cF^{\vW(0),\va}_{R,B_i} \right)
	 	\le 	\frac{B_i}{\sqrt{n}} + \frac{2 R(R + \frac{C n_s}{\sqrt{m\delta}\lambdas}) \sqrt m}{ \kappa} + R \sqrt{2\log \frac{10}{\delta}} .
	 	\end{align*}
	 \end{enumerate}
 
 Let $i^*$ be the smallest integer such that $B\le B_{i^*}$. Then we have $i^* \le O\left( \sqrt{\frac{n}{\lambda_0}} \right)$ and $B_{i^*} \le B + 1$.
From above we know $f_{\vW(t)}\in \cF_{R, B_{i^*}}^{\vW(0), \vect a}$, and
\begin{align*}
     \calR_S\left( \cF^{\vW(0),\va}_{R,B_{i^*}} \right) 
     \le \,&	\frac{B+1}{\sqrt{n}} + \frac{2 R(R + \frac{C n_s}{\sqrt{m\delta}\lambdas}) \sqrt m}{ \kappa} + R \sqrt{2\log \frac{10}{\delta}} \\
     = \,& \frac{\sqrt{\tilde{\vy}^\top \left(\vH^{\infty}\right)^{-1}\tilde{\vy}} }{\sqrt{n}} + \frac{1}{\sqrt{n}} + \frac{\poly\left( \norm{\tilde{\vy}}_2, n, n_s, \frac{1}{\lambda_0}, \frac{1}{\lambda_{0_s}}, \frac{1}{\delta} \right)}{m^{1/4} \kappa^{1/2}} + \frac{2 R(R + \frac{C n_s}{\sqrt{m\delta}\lambdas}) \sqrt m}{ \kappa} + R \sqrt{2\log \frac{10}{\delta}} \\
     \le \,& \sqrt{\frac{\tilde{\vy}^\top \left(\vH^{\infty}\right)^{-1}\tilde{\vy}}{n}} + \frac{1}{\sqrt{n}} + \frac{\poly\left( \norm{\tilde{\vy}}_2, n, n_s, \frac{1}{\lambda_0}, \frac{1}{\lambda_{0_s}}, \frac{1}{\delta} \right)}{m^{1/4} \kappa^{1/2}} \le \sqrt{\frac{\tilde{\vy}^\top \left(\vH^{\infty}\right)^{-1}\tilde{\vy}}{n}}   + \frac{2}{\sqrt{n}} .
\end{align*}
 
 Next, from the theory of Rademacher complexity and a union bound over a finite set of different $i$'s, for any random initialization $\left(\mat W(\init), \mat a \right)$, with probability at least $1-\delta/3$ over the sample $S$, we have
 \begin{align*}
 \sup_{f \in \cF_{R, B_i}^{\vW(0), \vect a}} \left\{ R(f)-L(f) \right\} \le 2 \calR_S\left(\cF_{R, B_i}^{\vW(0), \vect a}\right) + O\left( \sqrt{\frac{\log\frac{n}{\lambda_0\delta}}{n}} \right) , \qquad \forall i\in \left\{1, 2, \ldots, O\left( \sqrt{\frac{n}{\lambda_0}} \right) \right\}.
 \end{align*}
 
 Finally, taking a union bound, we know that with probability at least $1-\frac23\delta$ over the sample $S$ and the random initialization $(\mat W(\init), \vect a)$, the followings are all satisfied (for some $i^*$):
 \begin{align*}
     &L(\vX, \vTheta(t)) \le \frac{1}{\sqrt n},\\
     &f\left(\cdot, \vTheta(t)\right) \in \cF_{R, B_{i^*}}^{\vW(0), \vect a} ,\\
     &\calR_S\left( \cF^{\vW(0),\va}_{R,B_{i^*}} \right) \le \sqrt{\frac{\tilde{\vy}^\top \left(\vH^{\infty}\right)^{-1}\tilde{\vy}}{n}}  + \frac{2}{\sqrt{n}}  , \\
     &\sup_{f \in \cF_{R, B_{i^*}}^{\vW(0), \vect a}} \left\{ R(f)-L(f) \right\} \le 2 \calR_S\left(\cF_{R, B_{i^*}}^{\vW(0), \vect a}\right) + O\left( \sqrt{\frac{\log\frac{n}{\lambda_0\delta}}{n}} \right)  .
 \end{align*}
 These together can imply:
 \begin{align*}
     R(\vTheta(t)) &\le \frac{1}{\sqrt n} + 2 \calR_S\left(\cF_{R, B_{i^*}}^{\vW(0), \vect a}\right) + O\left( \sqrt{\frac{\log\frac{n}{\lambda_0\delta}}{n}} \right) \\
     &\le \frac{1}{\sqrt n} + 2 \left(\sqrt{\frac{\tilde{\vy}^\top \left(\vH^{\infty}\right)^{-1}\tilde{\vy}}{n}}  + \frac{2}{\sqrt{n}}    \right) + O\left( \sqrt{\frac{\log\frac{n}{\lambda_0\delta}}{n}} \right)  \\
      &= 2\sqrt{\frac{\tilde{\vy}^\top \left(\vH^{\infty}\right)^{-1}\tilde{\vy}}{n}}  + O\left( \sqrt{\frac{\log\frac{n}{\lambda_0\delta}}{n}} \right) .
 \end{align*}
 This completes the proof.
\end{proof}

\subsection{Linear teachers: Proof of corollary 6.3}
We now consider the case where
\begin{align*}
    g_S(\vx) = \vx^\top \thetas, \quad g_T(\vx) = \vx^\top \thetat,
\end{align*}
which is the case in \cref{cor:shallow_relu_linear_bound}.

We will start with stating the random initialization population risk bound for this case, which we will compare our result to:
\begin{corollary}[Population risk bound for random initialization from \cite{arora2019fine}]\label{cor:shallow_vanilla_generalization}
Assume that the random initialized model with weights $\vTheta(t)$ was trained according to Theorem 5.1 from \cite{arora2019fine} and that $\vy = \vX \thetat$, then with probability $1-\delta$
   \begin{equation}\label{eq:bound_relu_thetat}
	R(\vTheta(t)) \le \frac{3\sqrt{2} \norm{\thetat}_2}{\sqrt{n}}   + O\left( \sqrt{\frac{\log\frac{n}{\lambda_0\delta}}{n}} \right) .
\end{equation}
\end{corollary}
This corollary is a direct result of plugging $\vy = \vX \thetat$ into Corollary 6.2 from \cite{arora2019fine}, and plugging the result into Theorem 5.1 from \cite{arora2019fine}.

As discussed in Section 6.1, we will assume that $f\left(\vX; \vTheta(0)\right) = \vX \thetas$. Since our model is non-linear, this assumption is not trivial, and requires some clarification.
For infinite width, Lemma 1 from \cite{xu2020neural} tells us that $n_S = 2d$ can suffice to achieve this, if the samples are chosen according to some conditions. For the case of finite width $m$, like is assumed in Theorem 6.1, no such equivalent exist. However, we can use \cref{cor:shallow_vanilla_generalization} for the pretraining, and achieve an $\epsilon$ bound on the pretraining population risk, for sufficiently large $n_S = \Omega\left(\frac{\norm{\thetas}^2}{\epsilon^2}\right)$. Then, approximate relaxations can be derived when we assume the two functions are $\epsilon$ close (i.e. $f\left(\vx, \vTheta(0)\right) = \vx^\top \thetas + \epsilon$).

We now restate our two corollaries from the main text:
\begin{customcor}{6.2}[Main Text]
    Suppose that $g_S(\vX) \triangleq \vX^\top \thetas$, $g_T(\vX) \triangleq \vX^\top \thetat$,
and assume Assumption \ref{asm:pretrain_supp} holds. Then, $\sqrt{\tilde{\vect y}^\top (\mat H^\infty)^{-1}\tilde{\vect y }} \le  3 \norm{\thetat - \thetas}_2.$
\end{customcor}
This is a direct corollary of Theorem 6.1 from \cite{arora2019fine} on $\vtildey$ defined above.

\begin{customcor}{6.3}[Main Text]\label{cor:shallow_relu_linear_bound}
    Under the conditions of Theorem 6.1 and Corollary 6.2, it holds that
	\begin{align*}
	    R(\vTheta(t)) \le \frac{6 \norm{\thetat - \thetas}_2}{\sqrt{n}}   + O\left( \sqrt{\frac{\log\frac{n}{\lambda_0\delta}}{n}} \right).
	\end{align*}
\end{customcor}

Comparing this to \cref{cor:shallow_vanilla_generalization} gives us the exact condition for when it is better to use fine-tuning instead of random initialization, which is
\begin{align*}
    \norm{\thetat - \thetas} < \frac{\norm{\thetat}}{\sqrt{2}}.
\end{align*}

We will now provide a proof for this results:
\begin{proof}[Proof of \cref{cor:shallow_relu_linear_bound}]
    In order to achieve this bound, we use the assumption on $f\left(\vX;\vTheta(0)\right)$, which gives us:
    \begin{align*}
        \vtildey = \vX\thetat - \vX\thetas = \vX(\thetat - \thetas).
    \end{align*}
    Hence, we can treat $\vtildey$ as if it was created by a linear label generation function $\thetat - \thetas$. Hence, by using Theorem 6.1 from \cite{arora2019fine} we can bound
    \begin{align*}
        \sqrt{\vtildey(\vH^\infty)^{-1}\vtildey} \le 3\norm{\thetat - \thetas}.
    \end{align*}
    Plugging this into Theorem 6.1 concludes the proof.
\end{proof}

%% file: ms.bbl
\begin{thebibliography}{10}

\bibitem{devlin2019bert}
Jacob Devlin, Ming-Wei Chang, Kenton Lee, and Kristina Toutanova.
\newblock Bert: Pre-training of deep bidirectional transformers for language
  understanding.
\newblock In {\em Proceedings of the 2019 Conference of the North American
  Chapter of the Association for Computational Linguistics: Human Language
  Technologies, Volume 1 (Long and Short Papers)}, pages 4171--4186, 2019.

\bibitem{lee2019latent}
Kenton Lee, Ming-Wei Chang, and Kristina Toutanova.
\newblock Latent retrieval for weakly supervised open domain question
  answering.
\newblock In {\em Proceedings of the 57th Annual Meeting of the Association for
  Computational Linguistics}, pages 6086--6096, 2019.

\bibitem{geva2020injecting}
Mor Geva, Ankit Gupta, and Jonathan Berant.
\newblock Injecting numerical reasoning skills into language models.
\newblock In {\em Proceedings of the 58th Annual Meeting of the Association for
  Computational Linguistics}, pages 946--958, 2020.

\bibitem{gururangan2020don}
Suchin Gururangan, Ana Marasovi{\'c}, Swabha Swayamdipta, Kyle Lo, Iz~Beltagy,
  Doug Downey, and Noah~A Smith.
\newblock Don’t stop pretraining: Adapt language models to domains and tasks.
\newblock In {\em Proceedings of the 58th Annual Meeting of the Association for
  Computational Linguistics}, pages 8342--8360, 2020.

\bibitem{peters2019tune}
Matthew~E Peters, Sebastian Ruder, and Noah~A Smith.
\newblock To tune or not to tune? adapting pretrained representations to
  diverse tasks.
\newblock In {\em Proceedings of the 4th Workshop on Representation Learning
  for NLP (RepL4NLP-2019)}, pages 7--14, 2019.

\bibitem{ZhangBHRV17}
Chiyuan Zhang, Samy Bengio, Moritz Hardt, Benjamin Recht, and Oriol Vinyals.
\newblock Understanding deep learning requires rethinking generalization.
\newblock In {\em 5th International Conference on Learning Representations,
  {ICLR} 2017, Toulon, France, April 24-26, 2017, Conference Track
  Proceedings}. OpenReview.net, 2017.

\bibitem{lyu2019gradient}
Kaifeng Lyu and Jian Li.
\newblock Gradient descent maximizes the margin of homogeneous neural networks.
\newblock In {\em International Conference on Learning Representations}, 2019.

\bibitem{woodworth2020kernel}
Blake Woodworth, Suriya Gunasekar, Jason~D Lee, Edward Moroshko, Pedro
  Savarese, Itay Golan, Daniel Soudry, and Nathan Srebro.
\newblock Kernel and rich regimes in overparametrized models.
\newblock In {\em Conference on Learning Theory}, pages 3635--3673. PMLR, 2020.

\bibitem{chizat2020implicit}
Lenaic Chizat and Francis Bach.
\newblock Implicit bias of gradient descent for wide two-layer neural networks
  trained with the logistic loss.
\newblock In {\em Conference on Learning Theory}, pages 1305--1338. PMLR, 2020.

\bibitem{wu2021direction}
Jingfeng Wu, Difan Zou, Vladimir Braverman, and Quanquan Gu.
\newblock Direction matters: On the implicit bias of stochastic gradient
  descent with moderate learning rate.
\newblock In {\em International Conference on Learning Representations}, 2021.

\bibitem{moroshko2020implicit}
Edward Moroshko, Blake~E Woodworth, Suriya Gunasekar, Jason~D Lee, Nati Srebro,
  and Daniel Soudry.
\newblock Implicit bias in deep linear classification: Initialization scale vs
  training accuracy.
\newblock {\em Advances in Neural Information Processing Systems}, 33, 2020.

\bibitem{arora2019implicit}
Sanjeev Arora, Nadav Cohen, Wei Hu, and Yuping Luo.
\newblock Implicit regularization in deep matrix factorization.
\newblock {\em Advances in Neural Information Processing Systems}, 32, 2019.

\bibitem{sarussi2021towards}
Roei Sarussi, Alon Brutzkus, and Amir Globerson.
\newblock Towards understanding learning in neural networks with linear
  teachers.
\newblock In {\em International Conference on Machine Learning}, 2021.

\bibitem{nagarajan2019generalization}
Vaishnavh Nagarajan and J~Zico Kolter.
\newblock Generalization in deep networks: The role of distance from
  initialization.
\newblock {\em arXiv preprint arXiv:1901.01672}, 2019.

\bibitem{li2020gradient}
Mingchen Li, Mahdi Soltanolkotabi, and Samet Oymak.
\newblock Gradient descent with early stopping is provably robust to label
  noise for overparameterized neural networks.
\newblock In {\em International Conference on Artificial Intelligence and
  Statistics}, pages 4313--4324. PMLR, 2020.

\bibitem{neyshabur2019towards}
Behnam Neyshabur, Zhiyuan Li, Srinadh Bhojanapalli, Yann LeCun, and Nathan
  Srebro.
\newblock Towards understanding the role of over-parametrization in
  generalization of neural networks.
\newblock In {\em International Conference on Learning Representations (ICLR)},
  2019.

\bibitem{dziugaite2017computing}
Gintare~Karolina Dziugaite and Daniel~M Roy.
\newblock Computing nonvacuous generalization bounds for deep (stochastic)
  neural networks with many more parameters than training data.
\newblock {\em arXiv preprint arXiv:1703.11008}, 2017.

\bibitem{jiang2019fantastic}
Yiding Jiang, Behnam Neyshabur, Hossein Mobahi, Dilip Krishnan, and Samy
  Bengio.
\newblock Fantastic generalization measures and where to find them.
\newblock {\em arXiv preprint arXiv:1912.02178}, 2019.

\bibitem{razin2020implicit}
Noam Razin and Nadav Cohen.
\newblock Implicit regularization in deep learning may not be explainable by
  norms.
\newblock In {\em Advances in Neural Information Processing Systems (NeurIPS)},
  2020.

\bibitem{neyshabur2020being}
Behnam Neyshabur, Hanie Sedghi, and Chiyuan Zhang.
\newblock What is being transferred in transfer learning?
\newblock {\em arXiv preprint arXiv:2008.11687}, 2020.

\bibitem{chua2021fine}
Kurtland Chua, Qi~Lei, and Jason~D Lee.
\newblock How fine-tuning allows for effective meta-learning.
\newblock {\em arXiv preprint arXiv:2105.02221}, 2021.

\bibitem{du2021fewshot}
Simon~Shaolei Du, Wei Hu, Sham~M. Kakade, Jason~D. Lee, and Qi~Lei.
\newblock Few-shot learning via learning the representation, provably.
\newblock In {\em International Conference on Learning Representations}, 2021.

\bibitem{mcnamara2017risk}
Daniel McNamara and Maria-Florina Balcan.
\newblock Risk bounds for transferring representations with and without
  fine-tuning.
\newblock In {\em International Conference on Machine Learning}, pages
  2373--2381. PMLR, 2017.

\bibitem{gunasekar2018characterizing}
Suriya Gunasekar, Jason Lee, Daniel Soudry, and Nathan Srebro.
\newblock Characterizing implicit bias in terms of optimization geometry.
\newblock In {\em International Conference on Machine Learning}, pages
  1832--1841. PMLR, 2018.

\bibitem{bartlett2020benign}
Peter~L Bartlett, Philip~M Long, G{\'a}bor Lugosi, and Alexander Tsigler.
\newblock Benign overfitting in linear regression.
\newblock {\em Proceedings of the National Academy of Sciences},
  117(48):30063--30070, 2020.

\bibitem{hastie2019surprises}
Trevor Hastie, Andrea Montanari, Saharon Rosset, and Ryan~J Tibshirani.
\newblock Surprises in high-dimensional ridgeless least squares interpolation.
\newblock {\em arXiv preprint arXiv:1903.08560}, 2019.

\bibitem{tsigler2020benign}
Alexander Tsigler and Peter~L Bartlett.
\newblock Benign overfitting in ridge regression.
\newblock {\em arXiv preprint arXiv:2009.14286}, 2020.

\bibitem{zou2021benign}
Difan Zou, Jingfeng Wu, Vladimir Braverman, Quanquan Gu, and Sham~M Kakade.
\newblock Benign overfitting of constant-stepsize sgd for linear regression.
\newblock {\em arXiv preprint arXiv:2103.12692}, 2021.

\bibitem{azulay2021implicit}
Shahar Azulay, Edward Moroshko, Mor~Shpigel Nacson, Blake Woodworth, Nathan
  Srebro, Amir Globerson, and Daniel Soudry.
\newblock On the implicit bias of initialization shape: Beyond infinitesimal
  mirror descent.
\newblock {\em arXiv preprint arXiv:2102.09769}, 2021.

\bibitem{yun2020unifying}
Chulhee Yun, Shankar Krishnan, and Hossein Mobahi.
\newblock A unifying view on implicit bias in training linear neural networks.
\newblock In {\em International Conference on Learning Representations}, 2020.

\bibitem{arora2019fine}
Sanjeev Arora, Simon Du, Wei Hu, Zhiyuan Li, and Ruosong Wang.
\newblock Fine-grained analysis of optimization and generalization for
  overparameterized two-layer neural networks.
\newblock In {\em International Conference on Machine Learning}, pages
  322--332. PMLR, 2019.

\bibitem{arora2018convergence}
Sanjeev Arora, Nadav Cohen, Noah Golowich, and Wei Hu.
\newblock A convergence analysis of gradient descent for deep linear neural
  networks.
\newblock In {\em International Conference on Learning Representations}, 2018.

\bibitem{davis1970rotation}
Chandler Davis and William~Morton Kahan.
\newblock The rotation of eigenvectors by a perturbation. iii.
\newblock {\em SIAM Journal on Numerical Analysis}, 7(1):1--46, 1970.

\bibitem{lecun-mnisthandwrittendigit-2010}
Yann LeCun and Corinna Cortes.
\newblock {MNIST} handwritten digit database.
\newblock 2010.

\bibitem{arora2018optimization}
Sanjeev Arora, Nadav Cohen, and Elad Hazan.
\newblock On the optimization of deep networks: Implicit acceleration by
  overparameterization.
\newblock In {\em International Conference on Machine Learning}, pages
  244--253. PMLR, 2018.

\bibitem{ji2019gradient}
Ziwei Ji and Matus Telgarsky.
\newblock Gradient descent aligns the layers of deep linear networks.
\newblock In {\em 7th International Conference on Learning Representations,
  ICLR 2019}, 2019.

\bibitem{vershynin2018high}
Roman Vershynin.
\newblock {\em High-dimensional probability: An introduction with applications
  in data science}, volume~47.
\newblock Cambridge university press, 2018.

\bibitem{pmlr-v9-glorot10a}
Xavier Glorot and Yoshua Bengio.
\newblock Understanding the difficulty of training deep feedforward neural
  networks.
\newblock In Yee~Whye Teh and Mike Titterington, editors, {\em Proceedings of
  the Thirteenth International Conference on Artificial Intelligence and
  Statistics}, volume~9 of {\em Proceedings of Machine Learning Research},
  pages 249--256, Chia Laguna Resort, Sardinia, Italy, 13--15 May 2010. PMLR.

\bibitem{du2018gradient}
Simon~S Du, Xiyu Zhai, Barnabas Poczos, and Aarti Singh.
\newblock Gradient descent provably optimizes over-parameterized neural
  networks.
\newblock In {\em International Conference on Learning Representations}, 2018.

\bibitem{NEURIPS2019_9015}
Adam Paszke, Sam Gross, Francisco Massa, Adam Lerer, James Bradbury, Gregory
  Chanan, Trevor Killeen, Zeming Lin, Natalia Gimelshein, Luca Antiga, Alban
  Desmaison, Andreas Kopf, Edward Yang, Zachary DeVito, Martin Raison, Alykhan
  Tejani, Sasank Chilamkurthy, Benoit Steiner, Lu~Fang, Junjie Bai, and Soumith
  Chintala.
\newblock Pytorch: An imperative style, high-performance deep learning library.
\newblock In H.~Wallach, H.~Larochelle, A.~Beygelzimer, F.~d\textquotesingle
  Alch\'{e}-Buc, E.~Fox, and R.~Garnett, editors, {\em Advances in Neural
  Information Processing Systems 32}, pages 8024--8035. Curran Associates,
  Inc., 2019.

\bibitem{harris2020array}
Charles~R. Harris, K.~Jarrod Millman, St{\'{e}}fan~J. van~der Walt, Ralf
  Gommers, Pauli Virtanen, David Cournapeau, Eric Wieser, Julian Taylor,
  Sebastian Berg, Nathaniel~J. Smith, Robert Kern, Matti Picus, Stephan Hoyer,
  Marten~H. van Kerkwijk, Matthew Brett, Allan Haldane, Jaime~Fern{\'{a}}ndez
  del R{\'{i}}o, Mark Wiebe, Pearu Peterson, Pierre G{\'{e}}rard-Marchant,
  Kevin Sheppard, Tyler Reddy, Warren Weckesser, Hameer Abbasi, Christoph
  Gohlke, and Travis~E. Oliphant.
\newblock Array programming with {NumPy}.
\newblock {\em Nature}, 585(7825):357--362, September 2020.

\bibitem{2020SciPy-NMeth}
Pauli Virtanen, Ralf Gommers, Travis~E. Oliphant, Matt Haberland, Tyler Reddy,
  David Cournapeau, Evgeni Burovski, Pearu Peterson, Warren Weckesser, Jonathan
  Bright, St{\'e}fan~J. {van der Walt}, Matthew Brett, Joshua Wilson, K.~Jarrod
  Millman, Nikolay Mayorov, Andrew R.~J. Nelson, Eric Jones, Robert Kern, Eric
  Larson, C~J Carey, {\.I}lhan Polat, Yu~Feng, Eric~W. Moore, Jake
  {VanderPlas}, Denis Laxalde, Josef Perktold, Robert Cimrman, Ian Henriksen,
  E.~A. Quintero, Charles~R. Harris, Anne~M. Archibald, Ant{\^o}nio~H. Ribeiro,
  Fabian Pedregosa, Paul {van Mulbregt}, and {SciPy 1.0 Contributors}.
\newblock {{SciPy} 1.0: Fundamental Algorithms for Scientific Computing in
  Python}.
\newblock {\em Nature Methods}, 17:261--272, 2020.

\bibitem{4160265}
John~D. Hunter.
\newblock Matplotlib: A 2d graphics environment.
\newblock {\em Computing in Science Engineering}, 9(3):90--95, 2007.

\bibitem{koltchinskii2017concentration}
Vladimir Koltchinskii and Karim Lounici.
\newblock Concentration inequalities and moment bounds for sample covariance
  operators.
\newblock {\em Bernoulli}, 23(1):110--133, 2017.

\bibitem{393720}
jlewk (https://mathoverflow.net/users/141760/jlewk).
\newblock Difference between identity and a random projection.
\newblock MathOverflow.
\newblock URL:https://mathoverflow.net/q/393720 (version: 2021-05-25).

\bibitem{xu2020neural}
Keyulu Xu, Mozhi Zhang, Jingling Li, Simon~S Du, Ken-ichi Kawarabayashi, and
  Stefanie Jegelka.
\newblock How neural networks extrapolate: From feedforward to graph neural
  networks.
\newblock {\em arXiv preprint arXiv:2009.11848}, 2020.

\end{thebibliography}
